\documentclass{article} % Anonymized submission
\usepackage{times}
\usepackage{cite}

\usepackage[utf8]{inputenc} % allow utf-8 input
\usepackage[T1]{fontenc}    % use 8-bit T1 fonts
\usepackage{hyperref}       % hyperlinks
\usepackage{url}            % simple URL typesetting
\usepackage{booktabs}       % professional-quality tables
\usepackage{amsfonts}       % blackboard math symbols
\usepackage{dsfont}
\usepackage{nicefrac}       % compact symbols for 1/2, etc.
\usepackage{microtype}      % microtypography

\usepackage{bm}

\usepackage{hyperref}
\usepackage{amsfonts}
\usepackage{multirow}

\usepackage{amsmath}
\usepackage{algorithm}
\usepackage{algorithmic}

\usepackage{graphicx}

\usepackage{amssymb}
\usepackage{tikz}
\usetikzlibrary{decorations.pathmorphing}

\newcommand{\reals}{\mathbb{R}}
\newcommand{\inner}[1]{\langle #1 \rangle}
\newcommand{\norm}[1]{\| #1 \|}
\newcommand{\mean}[2]{\mathbb{E}_{#1} \left[ #2 \right]}
\newcommand{\wvec}[1]{\bm{w}^{(#1)}}
\newcommand{\kvec}[1]{\bm{k}^{(#1)}}

\newcommand{\xvec}[1]{\bm{x}^{(#1)}}

\newcommand{\p}{^{\prime}}

\newtheorem{theorem}{Theorem}

\newtheorem{lemma}{Lemma}
\newtheorem{corollary}{Corollary}
\newcommand{\BlackBox}{\rule{1.5ex}{1.5ex}}  % end of proof
\newenvironment{proof}{\par\noindent{\bf Proof\ }}{\hfill\BlackBox\\}

\newcommand{\algref}[1]{Algorithm~\ref{#1}}
\newcommand{\lemref}[1]{Lemma~\ref{#1}}
\newcommand{\figref}[1]{Figure~\ref{#1}}
\newcommand{\thmref}[1]{Theorem~\ref{#1}}
\newcommand{\crlref}[1]{Corollary~\ref{#1}}
\newcommand{\appref}[1]{Appendix~\ref{#1}}
\newcommand{\secref}[1]{Section~\ref{#1}}

\DeclareMathOperator{\sign}{sign}

\usepackage{authblk}

\title{A Provably Correct Algorithm for Deep Learning that
  Actually Works}
  
\author{
  Eran Malach
  }
\author{
  Shai Shalev-Shwartz
  }
\affil{School of Computer Science, The Hebrew University, Israel}
\date{}

\begin{document}

\maketitle
\thispagestyle{empty}

\begin{abstract}
We describe a layer-by-layer algorithm for training deep convolutional
networks, where each step involves gradient updates for a two layer
network followed by a simple clustering algorithm. Our algorithm stems
from a deep generative model that generates images level by level,
where lower resolution images correspond to latent semantic
classes. We analyze the convergence rate of our algorithm assuming
that the data is indeed generated according to this model (as well as
additional assumptions). While we do not pretend to claim that the
assumptions are realistic for natural images, we do believe that they
capture some true properties of real data. Furthermore, we show that
our algorithm actually works in practice (on the CIFAR dataset), achieving
results in the same ballpark as that of vanilla convolutional neural
networks that are being trained by stochastic gradient
descent. Finally, our proof techniques may be of independent interest. 
\end{abstract}

\section{Introduction}
The success of deep convolutional neural networks (CNN) has sparked
many works trying to understand their behavior.  We can roughly
separate these works into three categories: First, the majority of the
works focus on providing various optimization methods and algorithms
that prove well in practice, but have almost no theoretical
guarantees. A second class of works focuses on analyzing practical
algorithms (mostly SGD), but under strong assumptions on the data
distribution, like linear separability or sampling from Gaussian
distribution, that often make these problems trivially solvable by
much simpler algorithms.  A third class of works takes less
restrictive assumptions on the data, provides strong theoretical
guarantees, but these guarantees hold for
algorithms that don't really work in practice.

In this work, we study a new algorithm for learning deep convolutional
networks, assuming the data is generated from some deep generative
model.  This model assumes that the examples are generated in a
hierarchical manner: each example (image) is generated by first
drawing a high-level semantic image, and iteratively refining the
image, each time generating a lower-level image based on the
higher-level semantics from the previous step. Similar models were
suggested in other works as good descriptions of natural images
encountered in real world data. These works, although providing
important insights, suffer from one of two major shortcomings: they
either suggest algorithms that seem promising for practical use, but
without any theoretical guarantees, or otherwise provide algorithms
with sound theoretical analysis that seem far from being applicable
for learning real-world data.

Our work achieves promising results in the following sense: first, we
show an algorithm along with a complete theoretical analysis, proving
it's convergence under the assumed generative model (as well as
additional, admittedly strong, assumptions). Second, we show that
implementing the algorithm to learn real-world data achieves
performance that are in the same ballpark as the popular CNN trained
with SGD-based optimizers. Third, the problem on which we apply our
algorithm is not trivially learned by simple ``shallow'' learning
algorithms. The main achievement of this paper is succeeding in all of
these goals together. As is usually the case in tackling hard
problems, our theoretical analysis makes strong assumptions on the
data distribution, and we clearly state them in our
analysis. Nevertheless, the resulting algorithm works on real data
(where the assumptions clearly do not hold).
That said, we do not wish to claim that such algorithm achieves state-of-the-art
results, and hence did not apply many of the common ``tricks'' that are 
used in practice to train a CNN, but rather compared
our algorithm to an ``out-of-the-box'' SGD-based optimization.

\section{Related Work}

As mentioned, we can roughly divide the works relevant to the scope of
this paper into three categories: (1) works that study practical algorithms (SGD) solving ``simple'' problems that can be otherwise learned with
``shallow'' algorithms. (2) works that study problems with less restrictive
assumptions, but using algorithms that are not applicable in practice.
(3) works that study a generative model similar to ours,
but either give no theoretical guarantees, or otherwise analyze
an algorithm that is ``tailored'' to learning the generative model,
and seems very far from algorithms used in practice.

Trying to study a practically useful algorithm, \cite{daniely2017sgd}
proves that SGD learns a function that approximates the best function
in the conjugate kernel space derived from the network
architecture. Although this work provides guarantees for a wide range
of deep architectures, there is no empirical evidence that the best
function in the conjugate kernel space performs at the same ballpark
as CNNs.  The work of \cite{andoni2014learning} shows guarantees on
learning low-degree polynomials, which is again learnable via SVM or
direct feature mapping. Other works study shallow (one-hidden-layer)
networks under some significant assumptions.  The works of
\cite{gori1992problem, brutzkus2017sgd} study the convergence of SGD
trained on linearly separable data, which could be learned with the
Perceptron algorithm, and the works of \cite{brutzkus2017globally,
  tian2017analytical,li2017convergence,zhong2017recovery} assume that
the data is generated from Gaussian distribution, an assumption that
clearly does not hold in real-world data. The work of
\cite{du2017convolutional} extends the results in
\cite{brutzkus2017globally}, showing recovery of convolutional kernels
without assuming Gaussian distribution, but is still limited to the
regime of shallow two-layer network.

Another line of work aims to analyze the learning of deep architectures,
in cases that exceed the capacity of shallow learning. The works of 
\cite{livni2014computational, zhang2015learning,zhang2016l1}
show polynomial-time algorithms aimed at learning deep models,
but that seem far from performing well in practice.
The work of \cite{zhang2016convexified} analyses a method of learning
a model similar to CNN which can be applied to learn multi-layer networks,
but the analysis is limited to shallow two-layer settings,  when
the formulated problem is convex.

Finally, there have been a few works suggesting distributional
assumptions on the data that are similar in spirit to the generative
model that we analyze in this paper. Again, these works can be largely
categorized into two classes: works that provide algorithms with
theoretical guarantees but no practical success, and works that show
practical results without theoretical guarantees. The work of
\cite{arora2014provable} shows a provably efficient algorithm for
learning a deep representation, but this algorithm seems far from
capturing the behavior of algorithms used in practice. Our approach
can be seen as an extension of the work of \cite{mossel2016deep}, who
studies Hierarchal Generative Models, focusing on algorithms and
models that are applicable to biological data. \cite{mossel2016deep}
suggests that similar models may be used to define image refinement
processes, and our work shows that this is indeed the case, while
providing both theoretical proofs and empirical evidence to this
claim.  Finally, the works of
\cite{tang2012deep,patel2016probabilistic,van2014factoring} study
generative models similar to ours, with promising empirical results
when implementing EM inspired algorithms, but giving no theoretical
foundations whatsoever.  

\section{Generative Model}
\label{sec:gen_model}
We begin by introducing our generative model. This model is based on
the assumption that the data is generated in a hierarchical
manner. For each label, we first generate a high-level semantic
representation, which is simply a small scale image, where each
``pixel'' represents a semantic class (in case of natural images,
these classes could be: background, sky, grass etc.).  From this
semantic image, we generate a lower level image, where each patch
comes from a distribution depending on each ``pixel'' of the
high-level representation, generating a larger semantic image (lower
level semantic classes for natural images could be: edges, corners,
texture etc.).  We can repeat this process iteratively any number of
times, each time creating a larger image of lower level semantic
classes.  Finally, to generate a greyscale or RGB image, we assume
that the last iteration of this process samples patches over
$\reals$. This model is described schematically in
\figref{fig:generative_model}, with a formal description given in
\secref{sec:full_model}. \secref{sec:synthetic} describes a synthetic
example of digit images generated according to this model.
\subsection{Formal Description}
\label{sec:full_model}

To generate an example, we start with sampling the label
$y \sim U(\mathcal{Y})$, where $U(\mathcal{Y})$ is the uniform
distribution over the set of labels. Given $y$, we generate a small
image with $m$ pixels, where each pixel belongs to a set
$\mathcal{C}_0$. Elements of $\mathcal{C}_0$ corresponds to semantic
entities (e.g. ``sky'', ``grass'', etc.). The generated image, denoted
$x^{(0)} \in \mathcal{C}_0^m$, is sampled according to some simple
distribution $\mathcal{D}_y$ (to be defined later). Next, we generate a new image,
$x^{(1)} \in \mathcal{C}_1^{ms}$ as follows. Pixel $i$ in $x^{(0)}$
corresponds to some $c \in \mathcal{C}_0$. For every such $c$, there
is a distribution $\mathcal{D}_c$ over $\mathcal{C}_1^s$, where we
refer to $s$ as a ``patch size''. So, pixel $i$ in  $x^{(0)}$ whose
value is $c \in \mathcal{C}_0$ generates a patch of size $s$ in
$x^{(1)}$ by sampling the patch according to $\mathcal{D}_c$. This
process continues, which yields images $x^{(2)}, \ldots, x^{(k)}$
whose sizes are $ms^2, \ldots, ms^k$, where each pixel in layer $i$
comes from $\mathcal{C}_i$. We assume that $\mathcal{C}_k = \reals$,
hence the final image is over the reals. The resulting example is the
pair $(x^{(k)},y)$. We denote the distribution generating the image
of level $i$ by $\mathcal{G}_i$.

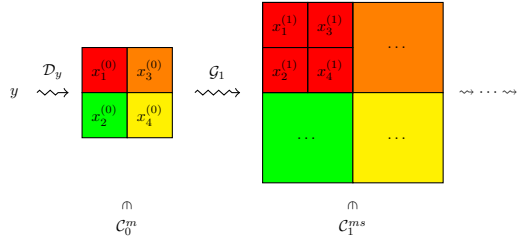
\begin{figure}
\center
\begin{tikzpicture}[scale=0.6, every node/.style={scale=0.6}]

\draw [black] (-0.5,5) node[] {$y$};

\draw [->, decorate, decoration={
    zigzag,
    segment length=4,
    amplitude=.9,post=lineto,
    post length=2pt
}] (0.0,5) -- (0.7,5) node[pos=0.5,above=5] {$\mathcal{D}_{y}$};

\draw [black] (2.0,2.5) node[rotate=270] {$\in$};
\draw [black] (2.0,2) node[] {$\mathcal{C}_0^m$};
\draw [black, fill=red] (1.0,5) rectangle (2.0,6) node[pos=.5] {\fontsize{10pt}{0}$x^{(0)}_{1}$};
\draw [black, fill=green] (1.0,4) rectangle (2.0,5) node[pos=.5] {\fontsize{10pt}{0}$x^{(0)}_{2}$};
\draw [black, fill=orange] (2.0,5) rectangle (3.0,6) node[pos=.5] {\fontsize{10pt}{0}$x^{(0)}_{3}$};
\draw [black, fill=yellow] (2.0,4) rectangle (3.0,5) node[pos=.5] {\fontsize{10pt}{0}$x^{(0)}_{4}$};

\draw [->, decorate, decoration={
    zigzag,
    segment length=4,
    amplitude=.9,post=lineto,
    post length=2pt
}] (3.5,5) -- (4.5,5) node[pos=0.5,above=5] {$\mathcal{G}_{1}$};

\def\x{5.0}
\def\y{3.0}
\def\s{2.0}

\draw [black] (\x+2,2.5) node[rotate=270] {$\in$};
\draw [black] (\x+2,2) node[] {$\mathcal{C}_1^{ms}$};
\draw [black, fill=red] (\x,\y+\s) rectangle (\x+\s,\y+2*\s);
\draw [black, fill=green] (\x,\y) rectangle (\x+\s,\y+\s) node[pos=.5] {\fontsize{10pt}{0}$\dots$};
\draw [black, fill=orange] (\x+\s,\y+\s) rectangle (\x+2*\s,\y+2*\s) node[pos=.5] {\fontsize{10pt}{0}$\dots$};
\draw [black, fill=yellow] (\x+\s,\y) rectangle (\x+2*\s,\y+\s) node[pos=.5] {\fontsize{10pt}{0}$\dots$};

\def\y{5.0}
\def\s{1.0}

\draw [black] (\x,\y+\s) rectangle (\x+1,\y+2*\s) node[pos=.5] {\fontsize{10pt}{0}$x^{(1)}_{1}$};
\draw [black] (\x,\y) rectangle (\x+\s,\y+\s) node[pos=.5] {\fontsize{10pt}{0}$x^{(1)}_{2}$};
\draw [black] (\x+\s,\y+\s) rectangle (\x+2*\s,\y+2*\s) node[pos=.5] {\fontsize{10pt}{0}$x^{(1)}_{3}$};
\draw [black] (\x+\s,\y) rectangle (\x+2*\s,\y+\s) node[pos=.5] {\fontsize{10pt}{0}$x^{(1)}_{4}$};

\def\x{8}

\draw [black] (\x+2.0,5) node {$\rightsquigarrow \dots \rightsquigarrow$};

\end{tikzpicture}

\caption{Generative model schematic description}\label{fig:generative_model}
\end{figure}

\subsection{Synthetic Digits Example}
\label{sec:synthetic}
To demonstrate our generative model, we use a small synthetic example to
generate images of digits. In this case, we use a three levels model,
where semantic classes represent lines, corners etc.
In the notations above, we use:
\begin{align*}
&\mathcal{C}_0 = \{
\includegraphics[scale=0.02]{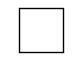},
\includegraphics[scale=0.02]{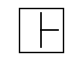},
\includegraphics[scale=0.02]{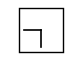},
\includegraphics[scale=0.02]{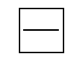},
\includegraphics[scale=0.02]{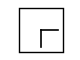},
\includegraphics[scale=0.02]{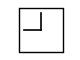},
\includegraphics[scale=0.02]{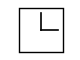},
\includegraphics[scale=0.02]{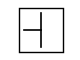},
\includegraphics[scale=0.02]{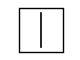}
\}
~~,~~ \mathcal{C}_1 = \{
\includegraphics[scale=0.02]{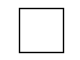},
\includegraphics[scale=0.02]{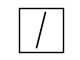},
\includegraphics[scale=0.02]{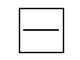},
\includegraphics[scale=0.02]{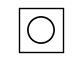},
\includegraphics[scale=0.02]{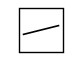},
\includegraphics[scale=0.02]{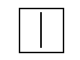}
\} ~~,~
&\mathcal{C}_2 = \reals
\end{align*}
We define the distributions
$\mathcal{D}_0, \dots, \mathcal{D}_9$
to be the distributions concentrated on the equivalent digital
representation:
\raisebox{-.35\height}{\includegraphics[scale=0.05]{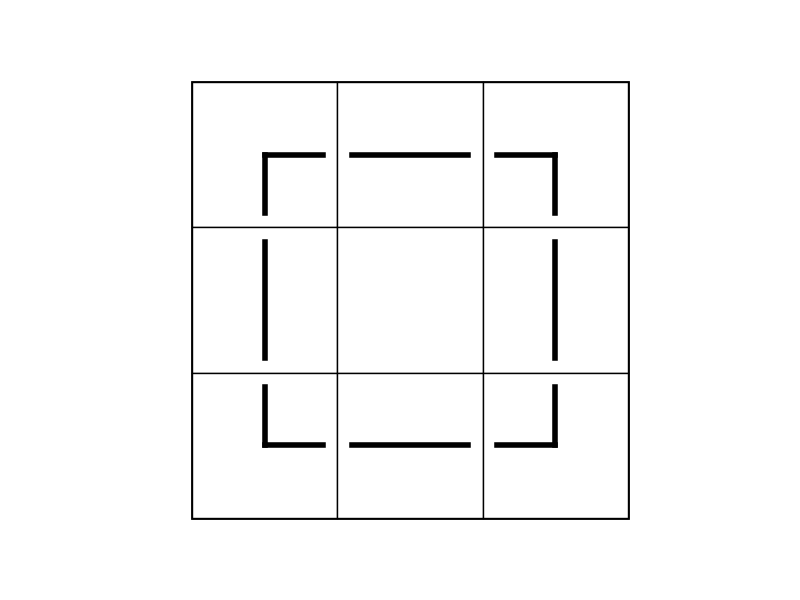}}~,
\raisebox{-.35\height}{\includegraphics[scale=0.05]{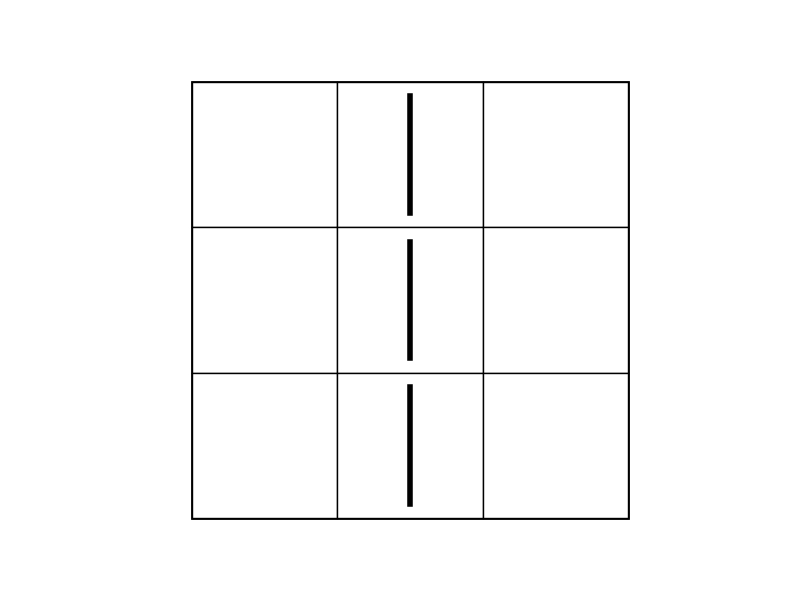}}~,
\raisebox{-.35\height}{\includegraphics[scale=0.05]{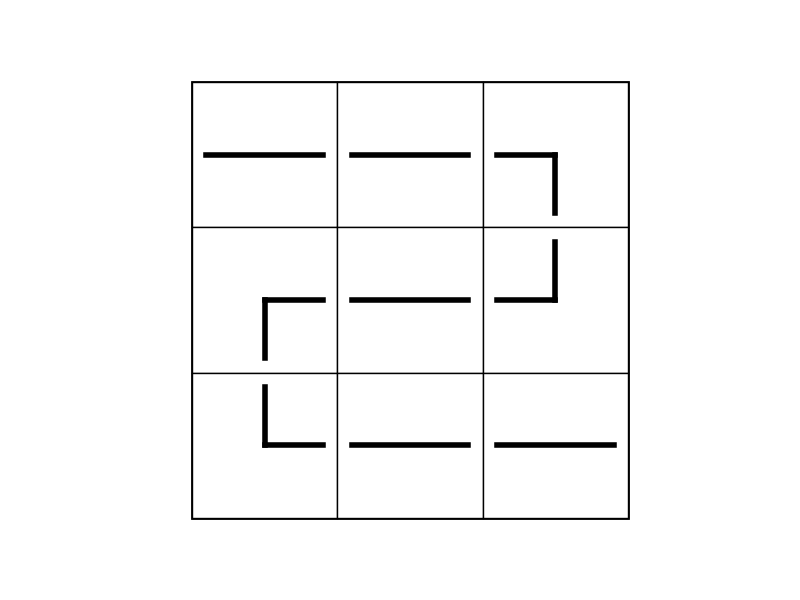}}~,
\raisebox{-.35\height}{\includegraphics[scale=0.05]{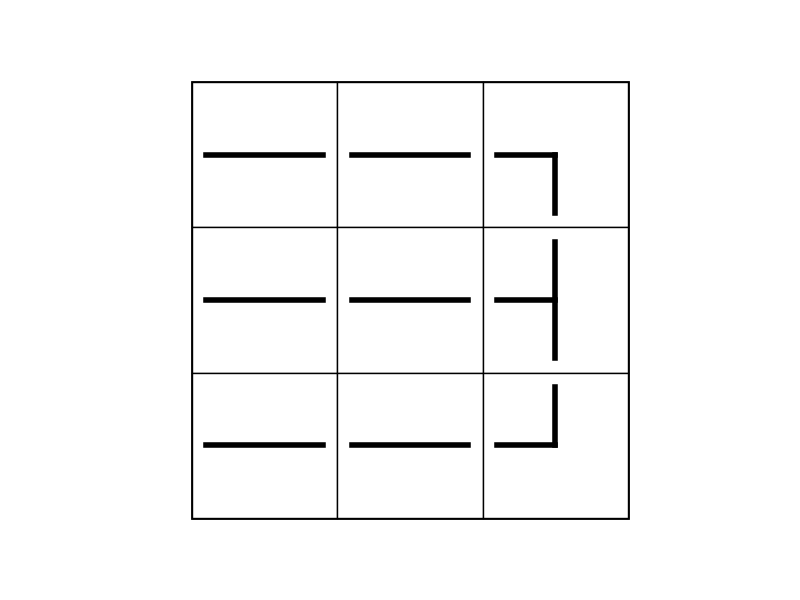}}~,
\raisebox{-.35\height}{\includegraphics[scale=0.05]{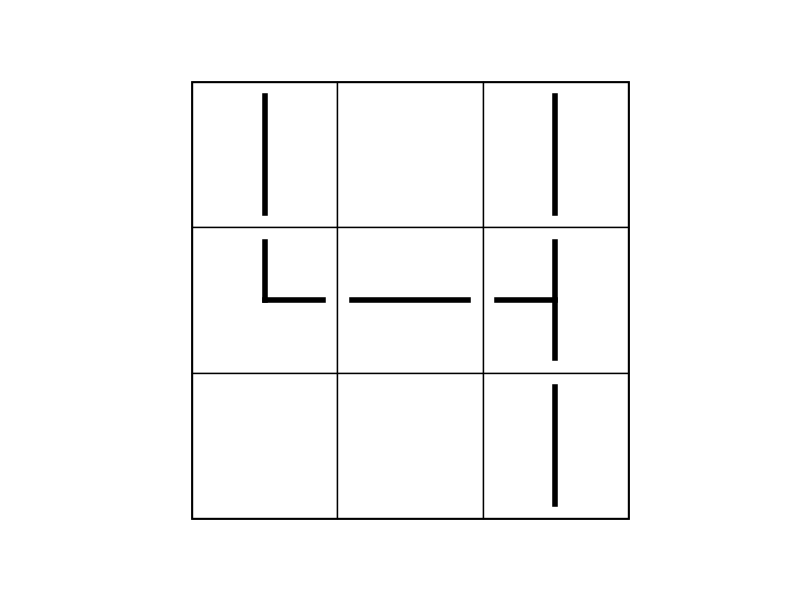}}~,
\raisebox{-.35\height}{\includegraphics[scale=0.05]{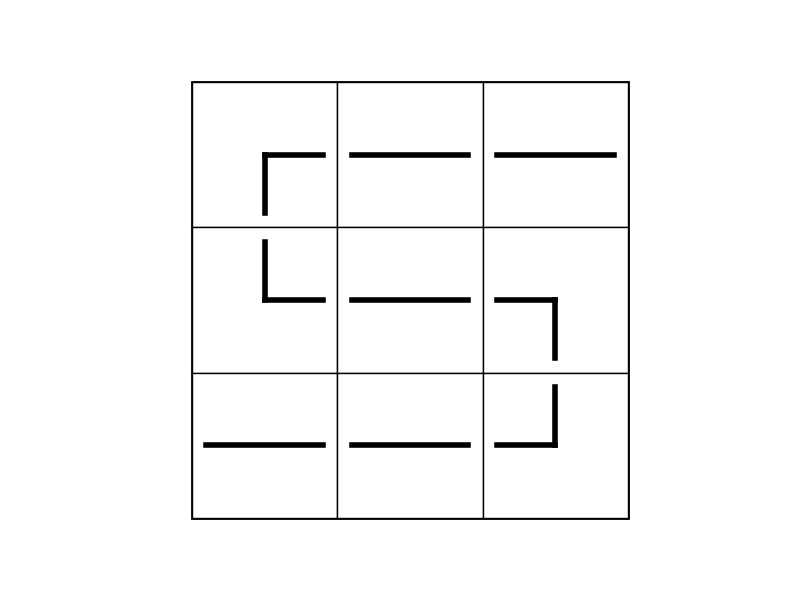}}~,
\raisebox{-.35\height}{\includegraphics[scale=0.05]{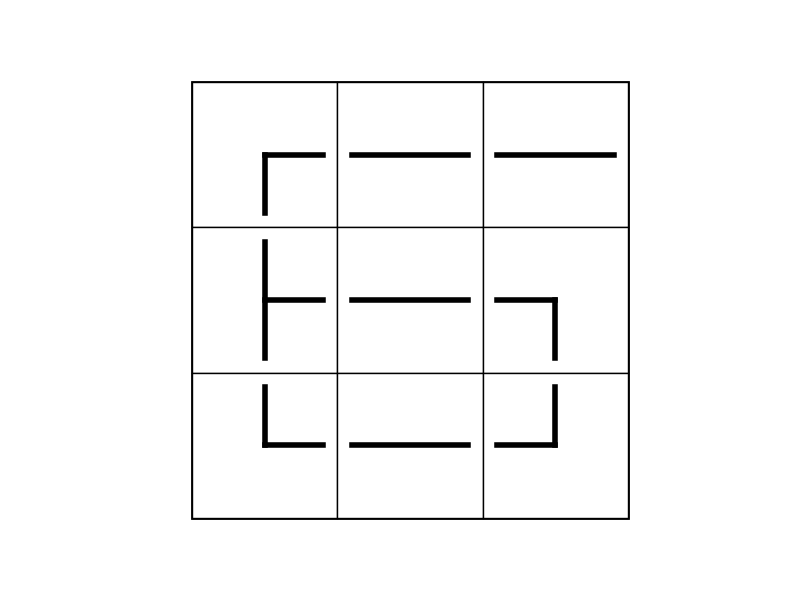}}~,
\raisebox{-.35\height}{\includegraphics[scale=0.05]{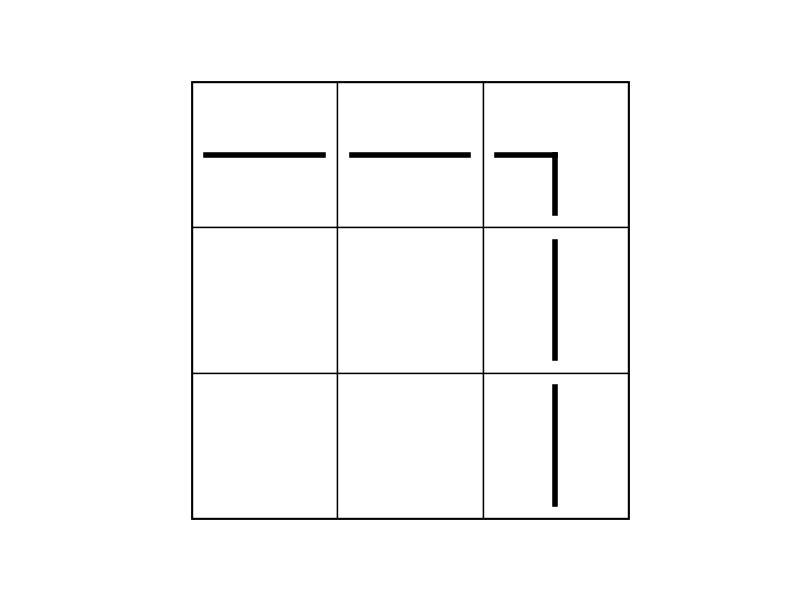}}~,
\raisebox{-.35\height}{\includegraphics[scale=0.05]{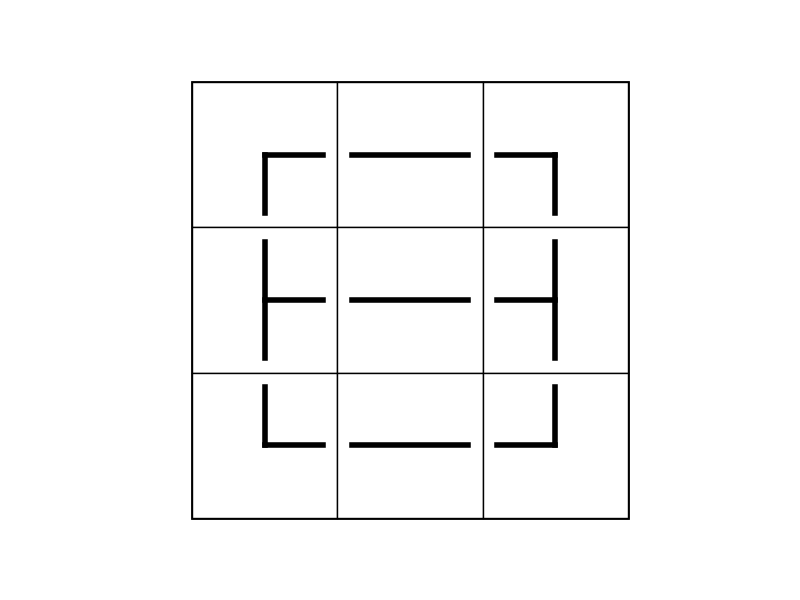}}~,
\raisebox{-.35\height}{\includegraphics[scale=0.05]{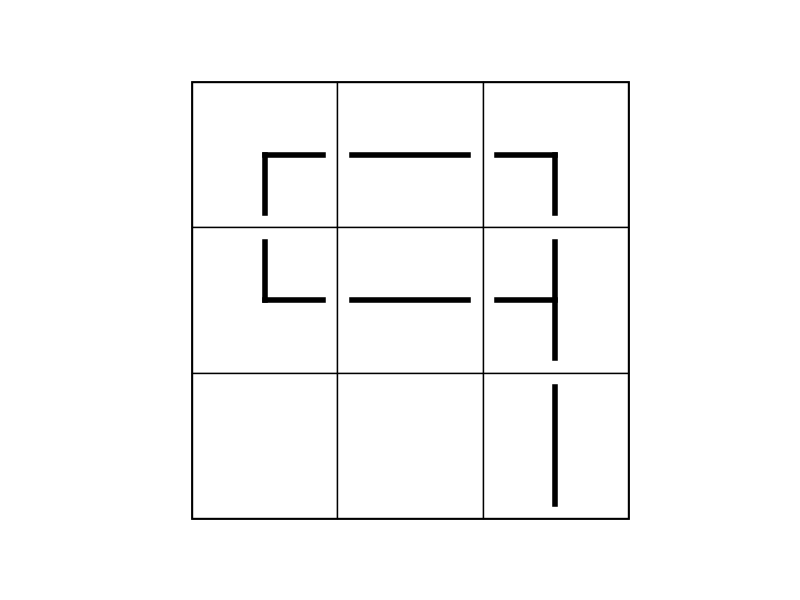}}~.

Now, in the second level of the generative model, each pixel
in $\mathcal{C}_0$ can generate one of four possible manifestations. For example, for the pixel
\raisebox{-.2\height}{\includegraphics[scale=0.03]{fig_sym0_2}}, we sample over:
\raisebox{-.35\height}{\includegraphics[scale=0.05]{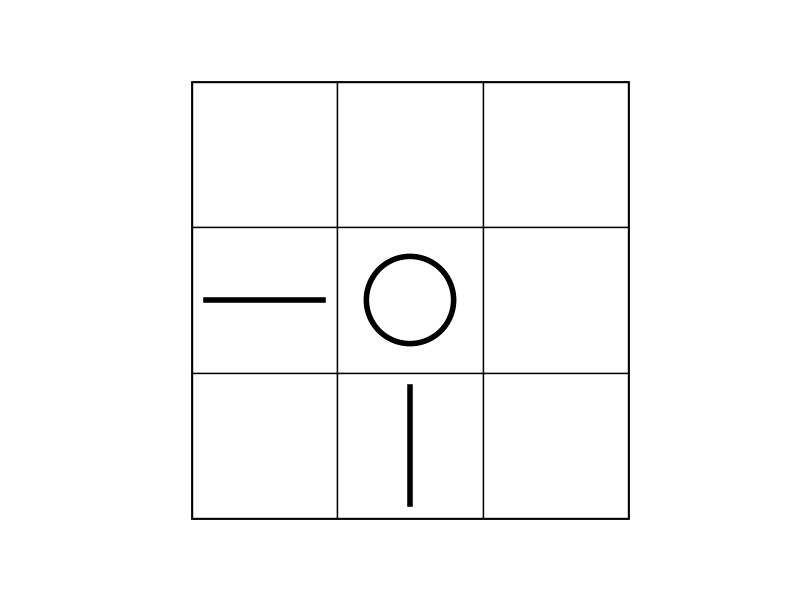}}~,
\raisebox{-.35\height}{\includegraphics[scale=0.05]{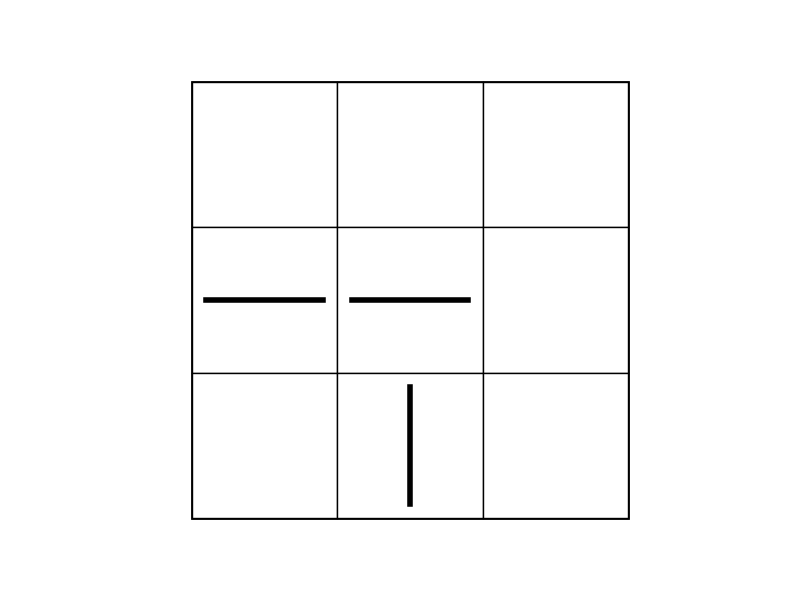}}~,
\raisebox{-.35\height}{\includegraphics[scale=0.05]{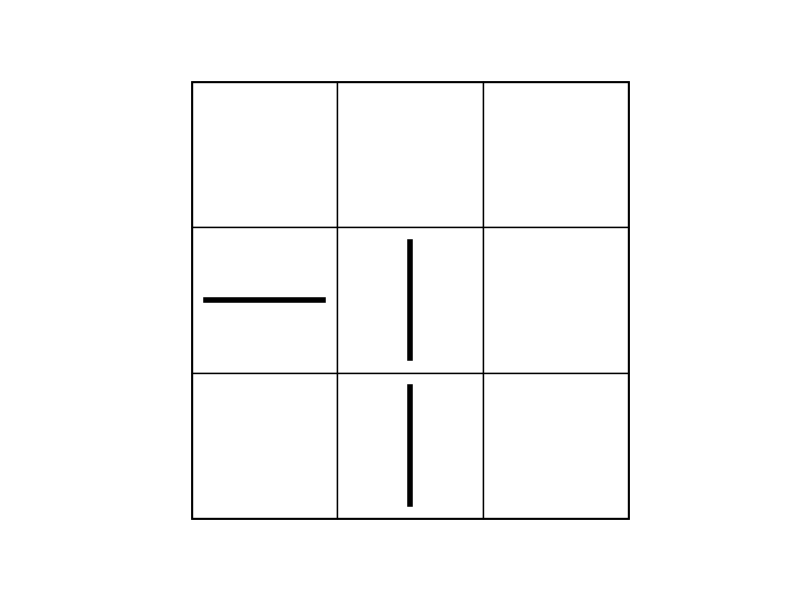}}~,
\raisebox{-.35\height}{\includegraphics[scale=0.05]{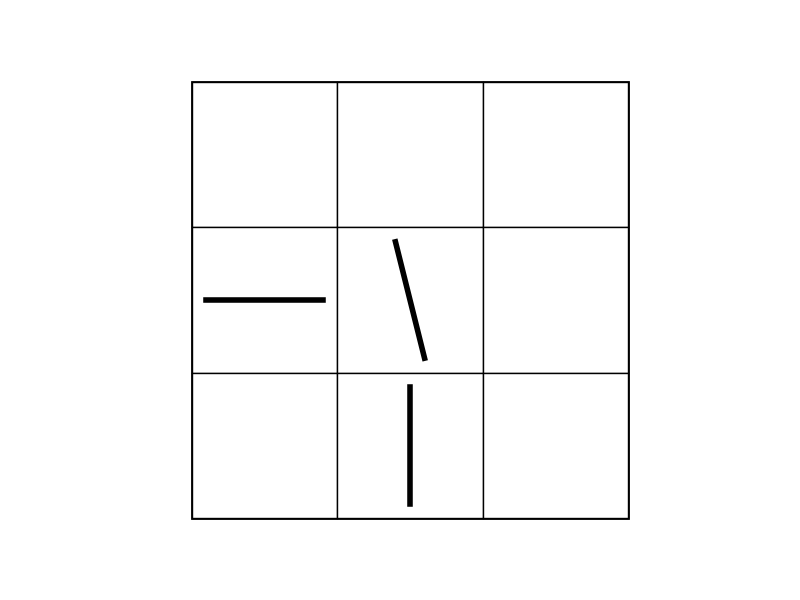}}~.
Similarly, in the final level we sample for each
$c \in \mathcal{C}_1$ from a distribution $\mathcal{D}_c$
supported over 4 elements. For example, for the pixel
\raisebox{-.2\height}{\includegraphics[scale=0.03]{fig_sym1_5}}, we sample over:
\raisebox{-.35\height}{\includegraphics[scale=0.05]{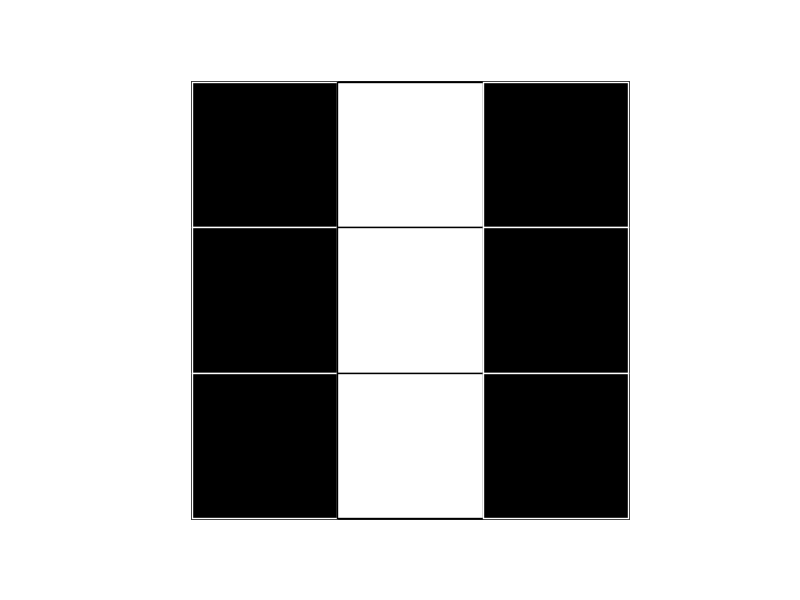}} ~,
\raisebox{-.35\height}{\includegraphics[scale=0.05]{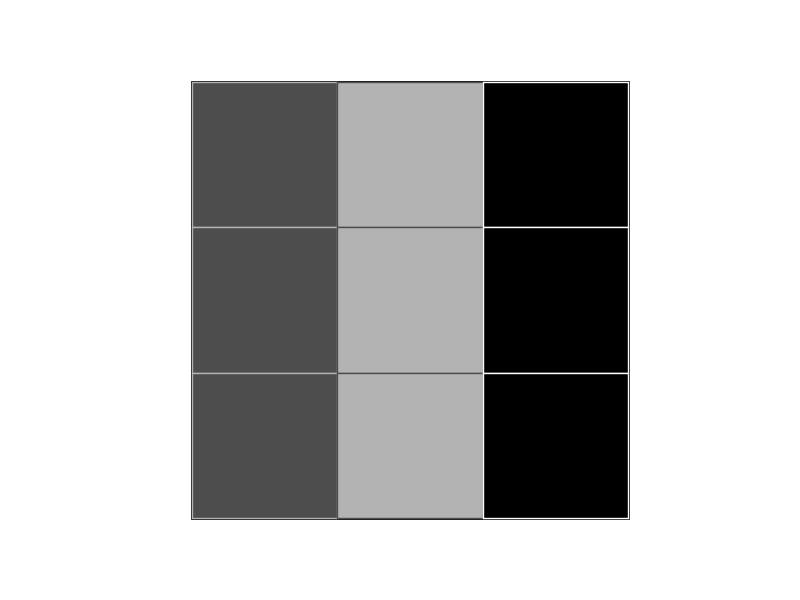}} ~,
\raisebox{-.35\height}{\includegraphics[scale=0.05]{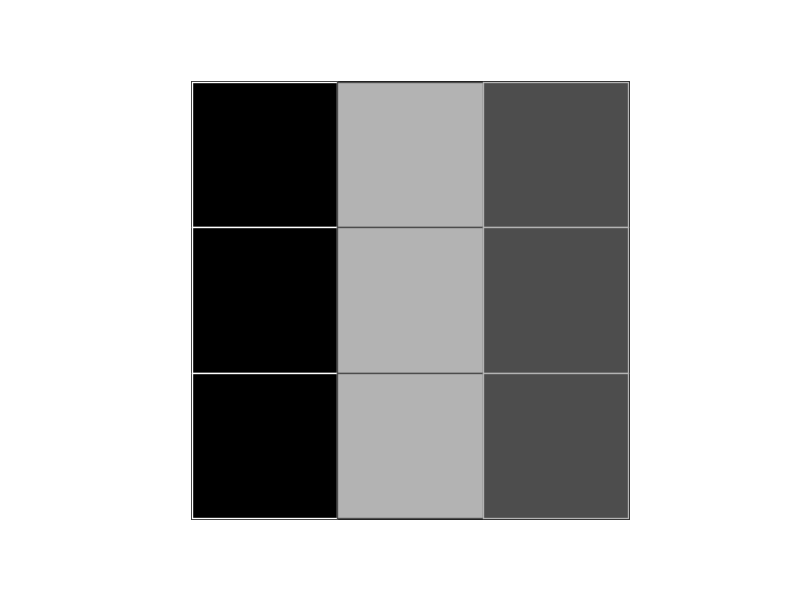}} ~,
\raisebox{-.35\height}{\includegraphics[scale=0.05]{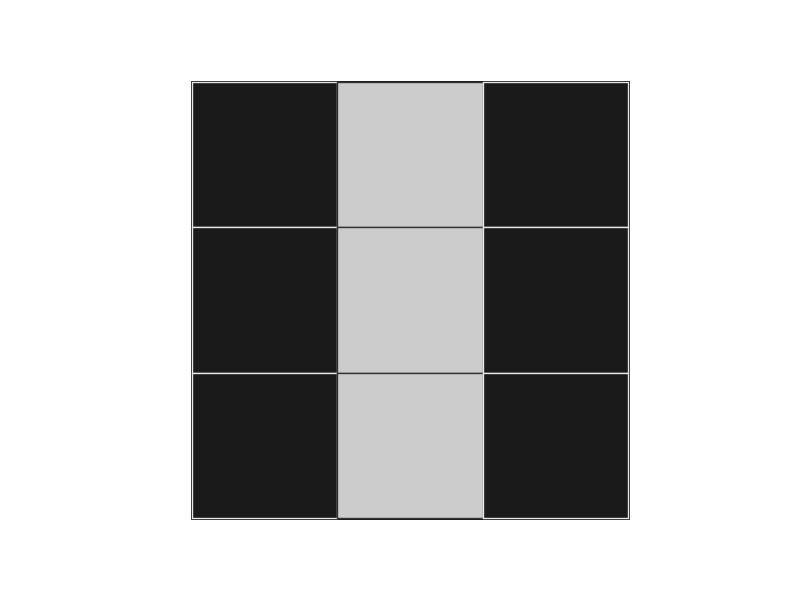}}~.

Notice that though this example is extremely simplistic,
it can generate $4^9$ examples per digit in the first
level, and $4^{90}$ examples for each digit in the final layer,
amounting to $9\cdot4^{90} \approx 1.38\cdot 10^{55}$ different examples.
Figure \ref{fig:synth_visualization} shows the process output.

\begin{figure}
\begin{tabular}{l@{\hspace{0.1cm}}r}
\includegraphics[scale=0.3]{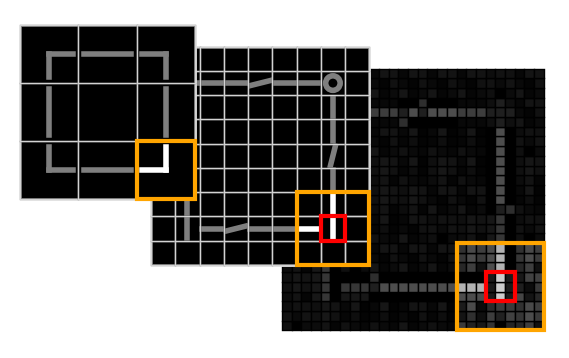} & 
\raisebox{0.1cm}{\includegraphics[scale=0.2]{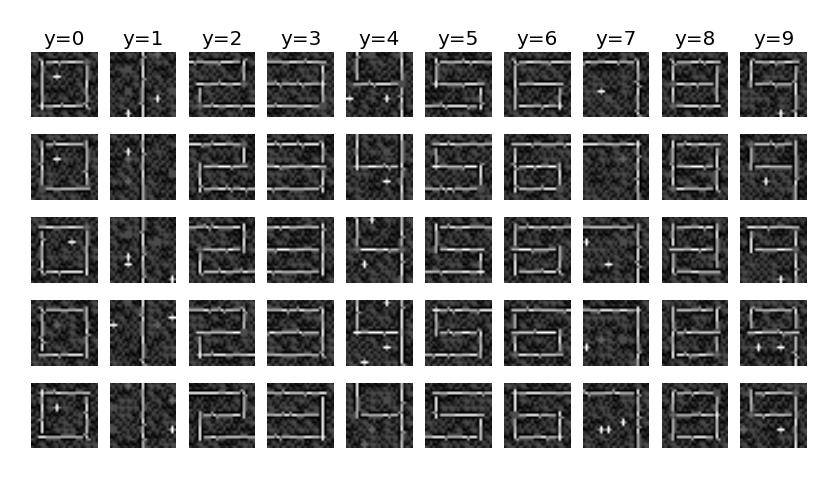}}
\end{tabular}
\caption{Left: Image generation process example. Right: Synthetic
  examples generated.} \label{fig:synth_visualization}
\end{figure}

\section{Algorithm}
Assume we are given data from the generative distribution described in
Section \ref{sec:gen_model}, our goal is to learn a classifier that
predicts the label for each image. A natural approach would be to try to learn
for each low-level patch, the semantic class (in the higher-level semantic image) from which it was generated. This way, we could cluster together
semantically related patches, exposing the higher-level semantic image
that generated the lower-level image. If we succeed in doing so multiple times,
we can infer the topmost semantic image in the hierarchy.
Assuming the high-level distribution $\mathcal{G}_0$ is simple enough
(for example, a linearly separable
distribution with respect to some embedding of the classes), 
we could then use a simple classification algorithm on the high-level image
to infer its label.

Unfortunately, we cannot learn these semantic classes directly
as we are not given access to the latent semantic images,
but only to the lowest-level image generated by the model.
To learn these classes, we use a combination of a simple clustering algorithm
and a gradient-descent based algorithm that learns a single layer
of a convolutional neural network. Surprisingly, as we show in the theoretical
section, the gradient-descent
finds an embedding of the patches such that patches from the same class
are close to each other, while patches from different classes are far away.
The clustering step then clusters together patches from the same class.

\subsection{Algorithm Description}
The algorithm we suggest is built from three building-blocks composed together to construct the full algorithm: (1) clustering
algorithm, (2) gradient-based optimization of two-layer Conv net
and (3) a simple classification algorithm.
In order to expose the latent representation of each layer
in the generative model, we perform the following iteratively: \\
\textbf{(1)} Run a centroid-based clustering algorithm on the patches
of size $\sqrt{s}\times \sqrt{s}$ from the input image
defined by the previous step (or the original image in the first
step),  w.r.t. the cosine distance, to get $\ell$ cluster centers. \\
\textbf{(2)} Run a convolution operation with the cluster centroids as
kernels, followed by ReLU with a fixed bias and a pooling operation.
This will result in
mapping the patches in the input images to (approximately)
orthogonal vectors in an intermediate space $\reals^\ell$. \\
\textbf{(3)} Initialize a 1x1 convolution operation, that maps from $\ell$
  channels into $n$ channels, followed by 
a linear layer that will output $|\mathcal{Y}|$ channels
(where it's input is the $n \times m \times m$
tensor flattened into a vector).
We train this two-layer subnet using a gradient-based
optimization method. As we show in the analysis,
this step implicitly learns an embedding of the patches
into a space where patches from the same semantic class
are close to each other, while patches from different
classes are far away, hence laying the ground for the clustering step
of the next iteration. \\
\textbf{(4)} ``Throw'' the last linear layer, thus leaving a trained
block of Conv$(\sqrt{s}\times \sqrt{s})$-ReLU-Pool-Conv$(1 \times 1)$ which
finds a ``good'' embedding of the patches of the input image,
and repeat the process again, where the output of this
block is the input to step 1.

Finally, after we perform this process for $k$ times,
we get a network of depth $k$ composed from 
Conv$(\sqrt{s}\times \sqrt{s})$-ReLU-Pool-Conv$(1 \times 1)$ blocks. Then, we feed
the output of this (already trained) network to some classifier,
training it to infer the label $y$ from the semantic representation that the convolution network outputs. This
training is done again using a gradient-based optimization
algorithm.
We now describe the building blocks for the algorithm,
followed by the definition of the complete algorithm.

\subsubsection{Clustering}
The first block of the algorithm is the clustering step.
We denote $\textrm{CLUSTER}_{\gamma}$ to be any polynomial time clustering algorithm,
such that given a sample $S \subseteq \reals^s$,
the algorithm outputs a mapping
$\phi_S: \reals^s \rightarrow \reals^{\ell}$,
satisfying that for every
$x_i,x_j \in S$, if $\norm{x_i-x_j} < \gamma$ then 
$\phi_S(x_i) = \phi_S(x_j)$,
and if $\norm{x_i-x_j} > 2\gamma$ then
$\phi_S(x_i) \perp \phi_S(x_j)$.
Notice that this clustering could be a trivial clustering algorithm:
for each example, we cluster together all the examples that are within
$\gamma$ distance from it, mapping different clusters to orthogonal 
vectors in $\reals^{\ell}$. Thus, we take $\ell$ to be the number
of clusters found in $S$.

For the consistency with common CNN architecture, we can use a centroid-based clustering algorithm that outputs the centroid of each cluster,
using these cetnroids
as kernels for a convolution operation. Combining this with ReLU
with a fixed bias and a pooling operation gives an operation that maps each patch to a single vector, where vectors of different patches are approximately orthogonal.

\subsubsection{Two-Layer Network Algorithm}
The second building-block of our main algorithm is a gradient-based optimization
algorithm that is used to train a two-layer convolutional subnet.
In this paper, we define a convolutional subnet to be a function $\mathcal{N}_{K,W}:
\reals^{\ell \times m} \rightarrow \reals$ defined by:
\[
\mathcal{N}_{K,W}(X) =
\inner{W^\top,K^\top X}
\]
Where we define the inner product between matrices $A,B$ as
$\inner{A,B} := \mathrm{tr}(A B)$.

This is equivalent to a convolution operation on an image,
followed by a linear weighted sum: assume $X$ is the matrix
where each column is a patch in an image (the ``im2col'' operation),
then multiplying this matrix by $K^\top$ is equivalent to performing a
convolution operation with kernels $\kvec{1}, \dots, \kvec{n}$ on the
original image (where we denote $\kvec{i}$ to be the $i$-th vector of
matrix $K$). Flattening the resulting matrix and multiplying by the
weights in $W$ yields the second linear layer.

The top linear layer of the network
outputs a prediction for the label $y \in \mathcal{Y}$,
and is trained with respect to the loss
$\mathcal{L}_{K,W}^S$ on a given set of examples $S$, defined as:
\[
\mathcal{L}^S_{K,W} =
\mean{(X,y) \sim S}{\ell_y(\mathcal{N}_{K,W}(X))}
\]
For some loss function
$\ell : \reals \times \mathcal{Y} \rightarrow \reals$.

After removing the top linear layer (which is used only to train the convolutional layer), this algorithm will output the matrix $K$.
This matrix is a set of 1x1 convolution kernels learned during the optimization, that are used on top of the previous operations. We denote
$\mathrm{TLGD}(S,T,\eta,n,\sigma)$ the algorithm that trains a
two-layer network of width $n$,
on sample $S$ for $T$ iterations with learning rate $\eta$,
randomly initializing $K,W$ from some distribution
with parameter $\sigma$ (described in details
in the theoretical section).
This algorithm outputs the Conv1x1 kernels learned.

As we show in our theoretical
analysis, running a gradient-based algorithm will implicitly learn an embedding
that maps patches from the same class to similar vectors, and patches from
different classes to vectors that are far away.

\subsubsection{Classification Algorithm}
Finally, the last building block of the algorithm is a classification
stage, that is used on top of the deep convolution architecture learned
in the previous steps. We consider some hypothesis space $\mathcal{H} \subset \mathcal{Y}^{\mathcal{X}}$ (for example linear separators).
Denote $\textrm{CLS}$ a polynomial time classification algorithm,
such that given a sample $S \subseteq \mathcal{X} \times \mathcal{Y}$, the algorithm outputs some hypothesis $\textrm{CLS}(S) \in \mathcal{H}$.
Again, we can assume this algorithm is trained using a gradient-based
optimization algorithm, to infer the label $y \in \mathcal{Y}$ based on 
the high-level semantics generated by the deep convolutional network
trained in the previous steps.

\subsubsection{Complete Algorithm}
Utilizing the building blocks described previously, 
our algorithm learns a deep CNN layer after layer. This network is used
to infer the label for each image. This algorithm is described formally
in \algref{alg:dc_full}. In the description, we use the notation
$\phi \ast A$ to denote the operation of applying a map $\phi: \mathcal{K}_0^{m_0} \rightarrow \mathcal{K}_1^{m_1}$ on a tensor $A$, replacing patches
of size $m_0$ by vectors in $\mathcal{K}_1^{m_1}$. Formally:
\[
\phi \ast A :=
[\phi(A_{:,i\cdot m_0 \dots (i+1) \cdot m_0})]_{i}
\]

\begin{algorithm}
   \caption{Deep Layerwise Clustering}\label{alg:dc_full}
\begin{algorithmic}
\footnotesize
  \STATE \textbf{input}: 
\begin{ALC@g}
  \STATE numbers $\gamma,\eta,T,n,\sigma$
  \STATE sample $S = \{(x_1,y_1), \dots (x_N, y_N)\} \subseteq \reals^{ms^k}\times \mathcal{Y}$
\end{ALC@g}
   \STATE $h_k \leftarrow id$
   \FOR {$\kappa = k \dots 1$}
   \STATE // construct sample of examples after running through the current network
   \STATE set $S_{\kappa} \leftarrow \{(h_{\kappa}(x_1),y_1), \dots,
   (h_{\kappa}(x_N),y_N)\}$
   \STATE // generate patches from the current sample, and cluster together
   \STATE set $P_{\kappa} \leftarrow \{\mathrm{patches ~of~ size~ \sqrt{s}\times \sqrt{s}~ from~ } S_{\kappa}\}$
   \STATE set $\phi_{\kappa} \leftarrow \textrm{CLUSTER}_{\gamma}
   (P_{\kappa})$
   \STATE // map the patches into orthogonal vectors using the embedding $\phi_{\kappa}$
   \STATE set $\hat{S}_{\kappa} \leftarrow \{(\phi_{\kappa} \ast h_{\kappa}(x_1),y_1), \dots, (\phi_{\kappa} \ast h_{\kappa}(x_N),y_N)\}$
   \STATE // train a two-layer network to find a ``good'' embedding for the
   patches
   \STATE set $K_{\kappa-1} \leftarrow \textrm{TLGD}
   (\hat{S}_{\kappa}, T, \eta, n, \sigma)$
   \STATE // add the current block to the network
   \STATE set $h_{\kappa-1} \leftarrow K_{\kappa-1}^\top (\phi_{\kappa} \ast h_{\kappa})$
   \ENDFOR
   \STATE // run a final clustering on the vectors of the current sample
   \STATE set $S_0 \leftarrow \{(h_0(x_1),y_1), \dots,
   (h_0(x_N),y_N)\}$
   \STATE set $P_0 \leftarrow \{\mathrm{patches ~of~ size~ 1\times1~ from~ } S_0\}$
   \STATE set $\phi_0 \leftarrow \textrm{CLUSTER}_{\gamma} (P_k)$
   \STATE // train the final classifier to predict the labels from the CNN
   output
   \STATE set $\hat{S}_0 \leftarrow \{(\phi_0 \ast h_0(x_1),y_1), \dots,
   (\phi_0 \ast h_0(x_N),y_N)\}$
   \STATE set $h \leftarrow \textrm{CLS}(\hat{S}_0)$
   \STATE return $h \circ h_{0}$
\end{algorithmic}
\end{algorithm}

%TODO add a schematic description of the algorithm?

\section{Theoretical Analysis}
In this section we prove that, under some assumptions, the algorithm
described in \algref{alg:dc_full} learns (with high probability) a
network model that correctly classifies the examples according to
their labels.  The structure of this section is as follows. We first
introduce our assumptions on the data distribution as well as on the
specific implementation of the algorithm. Next, we turn to 
the analysis of the algorithm itself, starting with showing that the sub module of
training a two-layer network implicitly learns an embedding of the
patches into a space where patches from a similar semantic class are
close to each other, while patches from different classes are far
apart. Using this property, we show that even a trivial clustering
algorithm manages to correctly cluster the patches. Finally, we prove
that performing these two steps (two-layer network + trivial clustering)
iteratively, layer by layer, leads to revealing the underlying model. 

\subsection{Assumptions}
\label{sec:assumptions}

Our analysis relies on several assumptions on the data distribution,
as well as on the suggested implementation of the algorithm.  These
assumptions are necessary for the theorems to hold, and admittedly are
far from being trivial. We believe that some of the assumptions can be
relaxed on the expense of a much more complicated proof.

\subsubsection{Distributional Assumptions} \label{sec:distributionAssumptions}
For simplicity, we focus on binary classification problems, namely,
$\mathcal{Y} = \{\pm 1\}$. The extension to multi-class problems is
straightforward. We assume that the sets of semantic classes
$\mathcal{C}_0, \dots ,\mathcal{C}_{k-1}$ are finite, and the final
(observed) image is over the reals, i.e $\mathcal{C}_k = \reals$.

We assume the following is true for all $\kappa \in [k]$.
For all $c \in \mathcal{C}_{\kappa}$
the distribution of the patches in the lower-level image for 
pixels of value $c$, denoted $\mathcal{D}_c$, is a uniform distribution over
a finite set of patches $S_{c,\kappa} \subseteq \mathcal{C}_{\kappa +
  1}^s$. We further assume that all these sets are disjoint and are of
fixed size $d$. 
\iffalse
For each $x\p \in S_{c,\kappa}$, we define a function $f^{\kappa}_{x\p}$
that takes as input a matrix of size $s \times m s^{\kappa-1}$ and
replaces each column of the image, denoted $x_i$, by the binary scalar
$\mathds{1}_{x_i = x\p}$. We slightly overload notation and allow the
input of $f^{\kappa}_{x\p}$ to be a vector of dimension $s m s^{\kappa-1} =
m s^{\kappa}$, where the meaning in this case is that we first reshape
the vector to match the dimension of the image, and then apply the
function on every column of the image.  Intuitively,  $f^{\kappa}_{x\p}$
replaces every patch of size $s$ in an image by a boolean that
indicates whether the patch equals to $x\p$. 
\fi

For every $c$ we denote by  $\mathcal{F}_c$ the operator that takes a
tensor (of some dimension) as its input and replaces every element of
the input by the boolean that indicates whether it equals to $c$. 

We introduce the notation:
$
v^{\kappa}_c
:= \mean{(z,y) \sim \mathcal{G}_{\kappa}}{-y \mathcal{F}_c(z)} \in \reals^{ms^{\kappa}}
$.

Notice that $-v^{\kappa}_c$ is the ``mean'' image over the
distribution for every given semantic class, $c$. For example,
semantic classes that tend to appear in the upper-left corner of the
image for positive images will have positive values in the upper-left
entries of $-v_c^{\kappa}$.  As will be explained later, these images
play a key role in our analysis.
% :we show that the activation of each patch after the training
% process is governed by the mean image of it's respective class, thus
% allowing us to cluster together patches from the same
% class. Therefore, the mean images serve as a ``semantic similarity''
% measure between different patches.

For our analysis to follow, we assume that the vectors
$\{v_c^\kappa\}_{c \in \mathcal{C}_\kappa}$ are linearly independent. 
For each $c_1, c_2 \in \mathcal{C}_\kappa$ we denote the angle between
$v^\kappa_{c_1}$ and $v^\kappa_{c_2}$ by:

\[
\angle(v^\kappa_{c_1}, v^\kappa_{c_2}) :=
\arccos \left( \frac{ v^\kappa_{c_1} \cdot v^\kappa_{c_2} }
{\norm{v^\kappa_{c_1}} \norm{v^\kappa_{c_2}}}\right) \in [0,\pi]
\]
Denote $\theta := \min_{\kappa < k, c_1,c_2 \in \mathcal{C}_\kappa}
\angle(v^\kappa_{c_1}, v^\kappa_{c_2})$ and
$\lambda := \min_{\kappa<k, c \in \mathcal{C}_\kappa} \norm{v^\kappa_c}$.
From the linear independence assumption it follows that both $\theta$
and $\lambda$ are strictly positive. The convergence of the algorithm
depends on $1/\theta$ and $1/\lambda$.

\subsubsection{High Level Efficient Learnability}

The essence of the problem to learn the mapping from the images in
$R^{ms^k}$ to the labels is that we do not observe the high level
semantic image in $\mathcal{C}_0^m$. To make this distinction clear,
we assume that, had we were given the semantic images in
$\mathcal{C}_0^m$, then the learning problem would have been
easy. Formally, there exists a classification algorithm,
denoted $\textrm{CLS}$, that upon receiving an
i.i.d. training set of polynomial size from the distribution
$\mathcal{G}_0$ over $\mathcal{C}_0^m \times \{\pm 1\}$, it returns
(with high probability, and after running in polynomial time) a
classifier whose error is at most $\epsilon$.

% \subsubsection{Separability Assumption}

% The main assumption that allows the algorithm to learn a correct model
% is that unlike our observed distribution $\mathcal{G}_k$, which can be
% very complicated, the high-level latent distribution $\mathcal{G}_0$ is fairly
% ``simple''. For that matter, recall that $\mathcal{G}_0$ is a
% distribution over the space $\mathcal{C}_0^m \times \{\pm 1\}$. We
% assume that $\mathcal{G}_0$ is separable in some sense that will be
% explained in the sequel.

% Recall that 
% given a tensor $X$ over a set $\mathcal{C}$ and a function $\phi :
% \mathcal{C} \to \reals^\ell$, we denoted by $\phi \ast X$ the tensor in
% which we replace every element $c$ of $X$ by the vector $\phi(c)$. 
% Given a sequence of examples $S = \{(x_1,y_1), \dots, (x_N,y_N)\}$ we
% denote by $\phi \ast S$ the application of $\phi$ to the instances in
% $S$, namely, $\phi \ast S := \{(\phi \ast x_1,y_1), \dots, (\phi \ast
% x_N,y_N)\}$. Finally, we denote by $\phi \ast \mathcal{G}$ the distribution generated by sampling
% $(x,y) \sim \mathcal{G}$ and returning $(\phi \ast x,y)$.

% Now, fix
% $\epsilon,\delta \in (0,1/4)$.
% We assume that for every mapping
% $\phi: \mathcal{C}_0 \rightarrow \reals^\ell$, such that $\mathrm{Im}\phi$ is an orthonormal set, 
% with probability at least $1-\delta$,
% the hypothesis returned by the classifier $h = CLS(\phi \ast S)$ satisfies
% \[
% P_{(x,y) \sim \phi \ast \mathcal{G}_0}
% (h(x) \ne y) < \epsilon 
% \]

\subsubsection{Assumptions on the Implementation of the Two
  Layers Building Block}
  \label{sec:assumptionGD}
For the analysis, we train the two-layer network with respect to the loss
$\ell_y(\hat{y}) = -y\hat{y}$. This loss simplifies
the analysis, and seems to capture a similar behavior to other loss types used in practice.

Although in practice we perform a variant of SGD on a sample of the data
to train the network, we perform the analysis with respect to the population
loss:
$\mathcal{L}_{K,W} =
\mean{(X,y) \sim \mathcal{G}}{\ell_y(\mathcal{N}_{K,W}(X))}$. 
We denote $K_t \in \reals^{\ell \times n}$ the weights of the first
layer of the network in iteration $t$,
and denote $W_0 \in \reals^{n \times m}$ the initial weights of the second
layer. For simplicity of the analysis, we assume that only the first layer of the network is trained, while the weights of the second layer are fixed.
Thus, we perform
the following update step at each iteration of the gradient descent:
$K_t = K_{t-1} - \eta \frac{\partial}{\partial K} \mathcal{L}_{K_{t-1},
  W_0}$. 
Applying this multiple times trains a network, denoted
$\mathcal{N}_{K_t, W_0}$.

As for the initialization of $K_0, W_0$, assume we initialize
each column of $K_0$ from a uniform distribution on a sphere
whose radius is at most $\frac{\sigma}{2\sqrt{n}}$,
where $\sigma$ is a parameter of the algorithm and $n$ is the number
of columns of $K_0$. 
We initialize $W_0 \sim \mathcal{N}(0,1)$.

\subsection{Two-Layer Algorithm}
\label{sec:two_layer}
In this part of the analysis we limit ourselves to observing the
properties of the two-layer network trained in the $\kappa$ iteration of
the main algorithm. Hence, we introduce a few simple notations to make
the analysis clearer. We first assume that we are given some mapping
from patches to $\reals^\ell$, denoted
$\phi: \mathcal{C}_{\kappa}^s \rightarrow \reals^\ell$, such that
$\mathrm{Im}(\phi)$ is a set of orthonormal vectors (this mapping is
learned by the previous steps of the algorithm, as we show in
\secref{sec:full_net_train}).  Assume we observe the distribution
$\mathcal{G} := \phi \ast \mathcal{G}_{\kappa}$. Recall that
$\mathcal{G}_{\kappa}$ is generated from the latent distribution
$\mathcal{G}_{\kappa-1}$ over higher level images.  We overload the
notation and use $m$ to denote the size of the semantic images from
the higher level of the model (namely, $m := ms^{\kappa-1}$). Thus,
$\mathcal{G}$ is a distribution over
$\reals^{\ell \times m} \times \{\pm 1\}$.

Now, we can ``forget'' the intermediate latent distributions
$\mathcal{G}_0, \dots, \mathcal{G}_{\kappa-2}$, and assume the
distribution $\mathcal{G}_{\kappa-1}$ is given by sampling
$y \sim U(\{\pm 1\})$ and then by sampling $(z,y) \sim \mathcal{D}_y$,
where $z \in \mathcal{C}_{\kappa-1}^m$ and we use $\mathcal{D}_y$ to
denote the distribution $\mathcal{G}_{\kappa-1}$ conditioned on
$y$. Finally, we denote $\mathcal{G}_z$ the distribution $\mathcal{G}$
conditioned on $z$. Thus, we can describe the sampling from
$\mathcal{G}$ schematically by:
\[
\rightsquigarrow^{U(\mathcal{Y})} y
\rightsquigarrow^{\mathcal{D}_y} (z,y) 
\rightsquigarrow^{\mathcal{G}_z} (X,y) 
\in \reals^{\ell \times m} \times \{\pm 1\}
\]

We denote $\mathcal{C} := \mathcal{C}_{\kappa-1}$, which is the set of
the semantic classes of the images in the latent distribution
$\mathcal{G}_{\kappa-1}$.  For every $c \in \mathcal{C}$ denote
$S_c := \{ \phi(p)\}_{p \in S_{c,\kappa-1}}$, which is the application
of $\phi$ on the set of patches in the observed distribution generated
from the semantic class $c$.  Notice that from the assumption on
$\phi$ it follows that $\cup_{c \in \mathcal{C}}S_c$ is a set of
orthonormal vectors.  Denote $v_c = v^{\kappa-1}_c$, the ``mean''
image of the semantic class $c$.  The following diagram describes the
process of generating $X$ from $z$, with $\tilde{x}$ being the example
generated by distribution $\mathcal{G}_{\kappa}$ (the example which is
embedded into the observed space $\reals^{\ell \times m}$):

\begin{align*}
\tiny
\underbrace{
\left[
\begin{matrix}
z_1 \\
\vdots \\
z_m
\end{matrix}
\right]}_{z \in \mathcal{C}^m}
\rightsquigarrow
\underbrace{
\left[
\begin{matrix}
\tilde{x}_1 \\
\vdots \\
\tilde{x}_s \\
\vdots \\
\tilde{x}_{ms-s+1} \\
\vdots \\
\tilde{x}_{ms}
\end{matrix}
\right]}_{\tilde{x} \in \mathcal{C}_\kappa^{ms}}
\overset{\mathrm{im2col}}{\mapsto}
\left[
\begin{matrix}
\tilde{x}_1 & \dots & \tilde{x}_{ms-s+1} \\
\vdots & & \vdots \\
\tilde{x}_s & \dots & \tilde{x}_{ms} \\
\end{matrix}
\right]
\overset{\phi}{\mapsto}
\underbrace{
\left[
\begin{matrix}
\overbrace{
\phi\left(
\begin{matrix}
\tilde{x}_1 \\
\vdots \\
\tilde{x}_s
\end{matrix}\right)}^{\xvec{1} \in S_{z_1}} & \dots &
\overbrace{
\phi \left(
\begin{matrix}
\tilde{x}_{ms-s+1} \\
\vdots \\
\tilde{x}_{ms}
\end{matrix}\right)}^{\xvec{m} \in S_{z_m}} 
\end{matrix}
\right]}_{X \in \reals^{\ell \times m}}
\end{align*}

Now, we can introduce the main theorem of this section.
This theorem states that training the two-layer Conv net
as defined previously will implicitly learn an embedding
of the observed patches into a space such that patches from the
same semantic class are close to each other, while patches
from different classes are far. Recall that we do not have access
to the latent distribution, and thus cannot possibly
learn such embedding directly. Therefore, this surprising property
of gradient descent is the key feature that allows
our main algorithm to learn the high-level semantics of the images.

\begin{theorem}
\label{thm:convergence}
Let $\theta,\lambda$ be as described in
\secref{sec:distributionAssumptions}. 
Assume we train a two-layer network of size $n > \frac{2\pi}{0.23\theta} \log (\frac{|\mathcal{C}|}{\delta})$ with respect to the population loss
on distribution $\mathcal{G}$, with learning rate $\eta$, for $T >
\frac{2 \sqrt{s} \sigma d}{\eta \lambda}$ iterations. Assume that the
training is as described in \secref{sec:assumptionGD}, where the
parameter $\sigma$ of the initialization is also described there. 
Then with probability of at least $1-\delta$:
\begin{enumerate}
\item for each $c \in \mathcal{C}$, for every $x_1, x_2 \in
  S_c$ we get
$\norm{K_T^\top \cdot x_1 - K_T^\top \cdot x_2} < \sigma$

\item for $c_1, c_2 \in \mathcal{C}$, if $c_1 \ne c_2$, for every
$x_1 \in S_{c_1}, x_2 \in S_{c_2}$, we get
$\norm{K_T^\top \cdot x_1 - K_T^\top \cdot x_2} > 2 \sqrt{s}\sigma$
\end{enumerate}
\end{theorem}

For lack of space, we give the full proof of this theorem in the appendix,
and give a rough sketch of the proof here: Observe the value of $\kvec{i}_t \cdot x\p$,
which is the activation of a kernel
in the first layer, denoted $\kvec{i}_t$, operated on a given patch
$x\p \in S_c$. Due to the gradient descent update rule, this value changes
by $- \eta x\p \cdot \frac{\partial}{\partial \kvec{i}} \mathcal{L}_{K_t, W_0}$
at each iteration. Analyzing this gradient shows that this expression,
i.e the change in the activation $\kvec{i}_t \cdot x\p$,
is in fact proportional to $\wvec{i}_0 \cdot v_c$.
In other words, the value of $\wvec{i}_0 \cdot v_c$ is the only factor that
dominates the behavior of the gradient with respect to $\kvec{i}$.
Hence, the activation of two patches generated from the same class
will behave similarly throughout the training process.
Furthermore, for patches from different classes
$c_1 \ne c_2$, if we happen to get: $\sign \left(\wvec{i}_0 \cdot
  v_{c_1} \right) \ne \sign \left(\wvec{i}_0 \cdot v_{c_1} \right)$,
due to the random initialization (which will happen in sufficient
probability), then the activations of patches
from class $c_1$ and patches from class $c_2$ will go in opposite directions,
and after enough iterations will be far apart.

To give more intuition as to why the proof works, we can look at the
whole process from a different angle. For two patches in an image sampled
from a given distribution, we can look at two measures of similarity:
First, we can observe a simple ``geometric'' similarity, like the $\ell_2$
distance between these two patches. Second, we can define a ``semantic''
similarity between patches to be the similarity between the distribution
of occurrences of each patch across the image (i.e, patches that
tend to appear more in the upper part of the image for positive labels
are in this sense ``semantically'' similar).
In our case, we show that the vector $v_c$ gives us exactly this measure
of similarity: two patches from the same class are semantically similar
in the sense that their mean distribution in the image is exactly 
the same image, denoted $v_c$. Given this notion, we can see 
why our full algorithm works: the clustering part of the
algorithm merges together geometrically similar patches, 
while the gradient descent algorithm maps semantically similar patches
to geometrically similar vectors, allowing the clustering of the next
iteration to perform clustering based again on the simple geometrical distance.
Note that while the technical proof heavily relies on our assumptions, the
intuitions above may hold true for real data.

\subsection{Full Network Training}
\label{sec:full_net_train}
In this section, we analyze the convergence of the full algorithm
described in \algref{alg:dc_full}, where our main claim is that
this algorithm successfully learns a model that classifies the examples
sampled from the observed distribution $\mathcal{G}_k$.
Formally, our main claim is given in the following theorem:
\begin{theorem}
\label{thm:full_alg}
Suppose that the assumptions given in \secref{sec:assumptions} hold. 
Fix $\epsilon, \delta \in (0,\frac{1}{4})$,
and let $\delta\p = \frac{\delta}{k+1}$.
Denote $C := \max_{\kappa < k} |\mathcal{C}_\kappa|$
the maximal number of semantic classes in each level $\kappa$. 
Let $\Delta$ denote the minimal distance between any two possible different patches in the observed images in $\reals^{m s^k}$. 
Choose $\gamma \le \frac{1}{2} \Delta$, $\sigma = \frac{\gamma}{\sqrt{s}}$,
$n > \frac{2\pi}{0.23\theta} \log (\frac{C}{\delta})$,
$T > \frac{2 \sqrt{s} \sigma d}{\eta \lambda}$.
Then, with probability $1-\delta$,
running \algref{alg:dc_full} with
parameters $\epsilon,\gamma,\eta,T,n,\sigma$
on data from distribution $\mathcal{G}_k$
returns hypothesis $h$ such that
$P_{(x,y) \sim \mathcal{G}_k} (h(x) \ne y) < \epsilon$.
\end{theorem}

To show this,
we rely on the result of Section \ref{sec:two_layer},
which guaranties that the embedding learned by the network
at each iteration maps patches from the same class to similar vectors.
Now, recall that our model assumes that a single pixel in a high-level image
is ``manifested'' as a patch in the lower-level image.
Thus, a patch of size $s$ in the higher-level image is manifested
in $s^2$ patch in the lower-level image, and many such manifestations are
possible. Thus, the fact that we find such ``good'' embedding allows
our simple clustering algorithm to cluster
together different low-level manifestations of a single high-level
patch. Hence, iteratively applying this embedding and clustering
steps allows to decode the topmost semantic image,
which can be then classified by our simple classification algorithm.

Before we show the proof, we remind a few notations that were used in
the algorithm's description. We use $\phi_\kappa$ to denote the clustering
of patches learned in the $\kappa$ iteration of the algorithm,
and $K_\kappa$ the weights of the kernels learned BEFORE the $\kappa$
step (thus, the patches mapped by $K_\kappa$ are the input to the clustering
algorithm that outputs $\phi_\kappa$). Note that in every step of the algorithm
we perform a clustering on patches of size $s$ in the current latent image, 
while at the last step we cluster only patches of size $1$ (i.e, cluster
the vectors in the ``channels'' dimension). This is because after
the final iteration we have a mapping of the distribution $\mathcal{G}_1$,
where patches of the same class are mapped to similar vectors. To generate
a mapping of $\mathcal{G}_0$, we thus only need to cluster these
vectors together, to get orthonormal representations of each class.
Finally, we use the notations $\phi \ast A$ to indicate that we operate
$\phi$ on every patch of the tensor $A$. When we use operations on
distributions, for example $h \circ \mathcal{G}$ or $\phi \ast \mathcal{G}$,
we refer to the new distribution generated by applying these operation to every
examples sampled from $\mathcal{G}$. The essence of the proof is the following lemma: 
\begin{lemma}
\label{lem:orth_mapping}
Let $\mathcal{G} := \mathcal{G}_k$ be the distribution over pairs
$(x,y)$, where $x$ is the observed image over the reals, and recall
that for $\kappa < k$, the distribution $\mathcal{G}_\kappa$ is over
pairs $(x,y)$ where $x$ is in a space of latent semantic images over
$\mathcal{C}_{\kappa}$. 
For every $\kappa \in [k]$,
with probability at least $1-\kappa \delta\p$,
there exists an orthonormal patch mapping
$\varphi_\kappa: \mathcal{C}_\kappa^s \rightarrow \reals^\ell$ such that
$\phi_\kappa \ast (h_\kappa \circ \mathcal{G}) = \varphi_\kappa \ast
\mathcal{G}_\kappa$, where $\phi_\kappa$ and $h_\kappa$ are as defined
in \algref{alg:dc_full}. 
\end{lemma}
The proofs of the lemma and of \thmref{thm:full_alg} are given in
\appref{app:proof_of_orth_mapping}. The lemma tells us that the neural network at step $\kappa$ of the
algorithm reveals (in some sense) the latent semantic structure. 

\section{Experiments}
As mentioned before, our analysis relies on distributional assumptions 
formalized in the generative model we suggest. A disadvantage of such
analyses is that the assumptions rarely hold for real-world
data, as the distribution of natural images is far more complex.
The goal of this section is to show that when running our algorithm on CIFAR-10, the performance of our model is in the same ballpark as
a vanilla CNN, trained with a common SGD-based optimization algorithm.
Hence, even though the data distribution deviates from our assumptions,
our algorithm still achieves good performance.

%\subsection{Implementation Details}
We chose the CIFAR-10 problem, as a rich enough dataset of natural images. As our aim is to show that our algorithm
achieves comparable result to a vanilla SGD-based optimization, and not to achieve
state-of-the-art results on CIFAR-10, we do not use any of the common
``tricks'' that are widely used when training deep networks
(such as data augmentation, dropout, batch normalization,
scheduled learning rate, averaging of weights across iterations etc.).
We implemented our algorithm by repeating the following steps twice:
% use an implementation that
% is inspired by the commonly used CNN configurations, specifically the one
% suggested in the Tensorflow implementation~\citep{tensorflow}. Our
% implementaion performs the following steps. 
(1) Sample $N$ patches of size 3x3 uniformly from the dataset.  (2)
For some $\ell$, run the K-means algorithm to find $\ell$ cluster
centers $c_1 \dots c_\ell$.  (3) At this step, we need to associate
each cluster with a vector in $\reals^\ell$, such that the image of
this mapping is a set of orthonormal vectors, and then map every patch
in every image to the vector corresponding to the cluster it belongs
to. We do so by performing Conv3x3 layer with the $\ell$ kernels
$c_1 \dots c_\ell$, and then perform ReLU operation with a fixed bias
$b$. This roughly maps each patch to the vector $e_i$, where $i$ is the
cluster the patch belongs to.  (4) While our analysis corresponds to
performing the convolution from the previous step with a stride of
$3$, to make the architecture closer to the commonly used CNNs
(specifically the one suggested in the Tensorflow
implementation~\cite{tensorflow}), we used a stride of $1$ followed
by a 2x2 max-pooling.  (5) Randomly initialize a two layered linear
network, where the first layer is Conv1x1 with $\ell\p$ output
channels, and the second layer is a fully-connected Affine layer that
outputs 10 channels to predict the 10 classes of CIFAR-10.  (6) Train the
two-layers with Adam optimization (\cite{adam}) on the cross-entropy loss, and 
remove the top layer. The output of the first layer is the output of
these steps.

Repeating the above steps twice yields a network with two blocks of
Conv3x3-ReLU-Pool-Conv1x1. We feed the output of these steps to a
final classifier that is trained again with Adam on cross entropy loss
for 100k iterations, to output the final classification of this
model. We test two choices for this classifier: a linear classifier
and a three-layers fully-connected neural network. Note that in both
cases, the output of our algorithm is a vanilla CNN. The only
difference is that it was trained differently.  To calibrate the
various parameters that define the model, we first perform random
parameter search, where we use 10k examples from the train set as
validation set (and the rest 40k as train set). After we found the
optimal parameters for all the setups we compare, we then train the
model again with the calibrated parameters on all the train data, and
plot the accuracy on the test data every 10k iterations. The parameters
found in the parameter search are listed in \appref{app:params}.

We compared our algorithm to several alternatives. First, the standard
CNN configuration in the Tensorflow implementation with two variants:
CNN+(FC+ReLU)$^3$ is the Tensorflow architecture and CNN+Linear is the
Tensorflow architecture where the last three fully connected layers
were replaced by a single fully connected layer. The goal of this
comparison is to show that the performance of our algorithm is in the
same ballpark as that of vanilla CNNs. Second, we use the same two
architectures mentioned before, but while using random weights for the
CNN and training only the FC layers. Some previous analyses of the
success of CNN claimed that the power of the algorithm comes from the
random initialization, and only the training of the last layer
matters. As is clearly seen, random weights are far from the
performance of vanilla CNNs. Our last experiment aims at showing the
power of the two layer training in our algorithm (step 6). To do so,
we compare our algorithm to a variant of it, in which step 6 is
replaced by random projections (based on Johnson-Lindenstrauss
lemma). We denote this variant by Clustering+JL.  As can be seen, this
variant gives drastically inferior results, showing that the training
step of Conv1x1 is crucial, and finds a ``good'' embedding for the
process that follows, as is suggested by our theoretical analysis.  A
summary of all the results is given in \figref{fig:results}.

\begin{figure}
\footnotesize
\center
\begin{tabular}{|l|l|l|l|l|l|l|l|l|}
  \hline
  Classifier & Accuracy(FC) & Accuracy(Linear) \\
  \hline
  CNN & \textbf{0.759} & 0.735 \\
  CNN(Random) & 0.645 & 0.616 \\
  Clustering+JL & 0.586 & 0.588 \\
  \hline
  Ours & \textbf{0.734} & 0.689 \\
  \hline
\end{tabular}
\caption{Results of various configurations on the CIFAR-10 dataset}\label{fig:results}.
\end{figure}

\paragraph{Acknowledgements:} This research is supported by the European Research Council (TheoryDL project).

\newpage

\bibliography{conv_net}{}

\begin{thebibliography}{10}

\bibitem{andoni2014learning}
Alexandr Andoni, Rina Panigrahy, Gregory Valiant, and Li~Zhang.
\newblock Learning polynomials with neural networks.
\newblock In {\em International Conference on Machine Learning}, pages
  1908--1916, 2014.

\bibitem{arora2014provable}
Sanjeev Arora, Aditya Bhaskara, Rong Ge, and Tengyu Ma.
\newblock Provable bounds for learning some deep representations.
\newblock In {\em International Conference on Machine Learning}, pages
  584--592, 2014.

\bibitem{brutzkus2017globally}
Alon Brutzkus and Amir Globerson.
\newblock Globally optimal gradient descent for a convnet with gaussian inputs.
\newblock {\em arXiv preprint arXiv:1702.07966}, 2017.

\bibitem{brutzkus2017sgd}
Alon Brutzkus, Amir Globerson, Eran Malach, and Shai Shalev-Shwartz.
\newblock Sgd learns over-parameterized networks that provably generalize on
  linearly separable data.
\newblock {\em arXiv preprint arXiv:1710.10174}, 2017.

\bibitem{daniely2017sgd}
Amit Daniely.
\newblock Sgd learns the conjugate kernel class of the network.
\newblock In {\em Advances in Neural Information Processing Systems}, pages
  2419--2427, 2017.

\bibitem{du2017convolutional}
Simon~S Du, Jason~D Lee, and Yuandong Tian.
\newblock When is a convolutional filter easy to learn?
\newblock {\em arXiv preprint arXiv:1709.06129}, 2017.

\bibitem{tensorflow}
G.~Google-Brain.
\newblock Tensorflow.
\newblock \url{https://www.tensorflow.org/tutorials/deep_cnn}, 2016.

\bibitem{gori1992problem}
Marco Gori and Alberto Tesi.
\newblock On the problem of local minima in backpropagation.
\newblock {\em IEEE Transactions on Pattern Analysis and Machine Intelligence},
  14(1):76--86, 1992.

\bibitem{adam}
Diederik~P. Kingma and Jimmy Ba.
\newblock Adam: {A} method for stochastic optimization.
\newblock {\em CoRR}, abs/1412.6980, 2014.

\bibitem{li2017convergence}
Yuanzhi Li and Yang Yuan.
\newblock Convergence analysis of two-layer neural networks with relu
  activation.
\newblock In {\em Advances in Neural Information Processing Systems}, pages
  597--607, 2017.

\bibitem{livni2014computational}
Roi Livni, Shai Shalev-Shwartz, and Ohad Shamir.
\newblock On the computational efficiency of training neural networks.
\newblock In {\em Advances in Neural Information Processing Systems}, pages
  855--863, 2014.

\bibitem{mossel2016deep}
Elchanan Mossel.
\newblock Deep learning and hierarchal generative models.
\newblock {\em arXiv preprint arXiv:1612.09057}, 2016.

\bibitem{patel2016probabilistic}
Ankit~B Patel, Minh~Tan Nguyen, and Richard Baraniuk.
\newblock A probabilistic framework for deep learning.
\newblock In {\em Advances in Neural Information Processing Systems}, pages
  2558--2566, 2016.

\bibitem{tang2012deep}
Yichuan Tang, Ruslan Salakhutdinov, and Geoffrey Hinton.
\newblock Deep mixtures of factor analysers.
\newblock {\em arXiv preprint arXiv:1206.4635}, 2012.

\bibitem{tian2017analytical}
Yuandong Tian.
\newblock An analytical formula of population gradient for two-layered relu
  network and its applications in convergence and critical point analysis.
\newblock {\em arXiv preprint arXiv:1703.00560}, 2017.

\bibitem{van2014factoring}
Aaron Van~den Oord and Benjamin Schrauwen.
\newblock Factoring variations in natural images with deep gaussian mixture
  models.
\newblock In {\em Advances in Neural Information Processing Systems}, pages
  3518--3526, 2014.

\bibitem{zhang2016l1}
Yuchen Zhang, Jason~D Lee, and Michael~I Jordan.
\newblock l1-regularized neural networks are improperly learnable in polynomial
  time.
\newblock In {\em International Conference on Machine Learning}, pages
  993--1001, 2016.

\bibitem{zhang2015learning}
Yuchen Zhang, Jason~D Lee, Martin~J Wainwright, and Michael~I Jordan.
\newblock Learning halfspaces and neural networks with random initialization.
\newblock {\em arXiv preprint arXiv:1511.07948}, 2015.

\bibitem{zhang2016convexified}
Yuchen Zhang, Percy Liang, and Martin~J Wainwright.
\newblock Convexified convolutional neural networks.
\newblock {\em arXiv preprint arXiv:1609.01000}, 2016.

\bibitem{zhong2017recovery}
Kai Zhong, Zhao Song, Prateek Jain, Peter~L Bartlett, and Inderjit~S Dhillon.
\newblock Recovery guarantees for one-hidden-layer neural networks.
\newblock {\em arXiv preprint arXiv:1706.03175}, 2017.

\end{thebibliography}
\bibliographystyle{plain}

\newpage

\appendix

\section{Proof of \thmref{thm:convergence}}
\label{sec:proof_conv}
For some class $c \in \mathcal{C}$ and for some patch $x\p \in S_c$,
denote $f_{x\p}$ a function that takes a matrix $X$ and returns a
vector $f_{x\p}(X)$ such that the $i$'th element of $f_{x\p}(X)$ is
the $1$ if the $i$'th column of $X$, denoted $\xvec{i}$, equals to $x\p$ and $0$
otherwise. That is, 
\[
f_{x\p} (X) := 
\left[
\begin{matrix}
\mathds{1}_{\xvec{1} = x\p} \\
\vdots \\
\mathds{1}_{\xvec{m} = x\p}
\end{matrix}
\right]
\]
Notice that from the orthonormality of the observed columns of $X$ it follows that:
$f_{x\p} (X)=X^{\top} x\p$.

We begin with proving the following technical lemma.

\begin{lemma}
\label{lem_vec}
For each $c \in \mathcal{C}$ and for each $x\p \in S_c$ we have:
\[
\mean{(X,y) \sim \mathcal{G}}{-y f_{x\p}(X)}  = \frac{1}{d} v_c
\]
\end{lemma}
\begin{proof}
Observe that
\[
\mean{(X,y) \sim \mathcal{G}}{-y f_{x\p}(X)}  = 
\frac{1}{2}\sum_{y=\pm1} -y 
\mean{z \sim \mathcal{D}_y}
{\mean{X \sim \mathcal{G}_{z}}
{f_{x\p} (X)}}
\]
Therefore, for each $j \in [m]$ we have:
\begin{align*}
\frac{1}{2}\sum_{y=\pm1} -y 
\mean{z \sim \mathcal{D}_y}
{\mean{X \sim \mathcal{G}_{z}}
{f_{x\p} (X)_j}}
&= \frac{1}{2}\sum_{y=\pm1} -y \mean{z \sim \mathcal{D}_y}
{\mean{x \sim \mathcal{G}_{z}}
{\mathds{1}_{\xvec{j} = x\p}}} \\
&= \frac{1}{2}\sum_{y=\pm1} -y \mean{z \sim \mathcal{D}_y}
{P_{x \sim \mathcal{G}_{z}}
(\xvec{j} = x\p)} \\
&= \frac{1}{2}\sum_{y=\pm1} -y \mean{z \sim \mathcal{D}_y}
{\frac{1}{d}\mathds{1}_{z_j = c}} \\
&= \frac{1}{d} \cdot \frac{1}{2}\sum_{y=\pm1} -y \mean{z \sim \mathcal{D}_y}
{\mathcal{F}_c (z)_j} = \frac{1}{d} [v_c]_j
\end{align*}
\end{proof}

The next lemma reveals a surprising connection between the gradient
and the vectors $v_c$. 
\begin{lemma}
\label{lem_grad}
for every
$c \in \mathcal{C}$ and for every $x\p \in S_c$:
\[
x\p \frac{\partial}{\partial \kvec{i}}
\mathcal{L}_{K,W} = 
\frac{1}{d} \wvec{i} \cdot v_c
\] 

\end{lemma}
\begin{proof}
For a fixed $X$ and  $W$, denote $\hat{y}(K) = \mathcal{N}_{K,W} (X)$.
 Note that:
\[
\frac{\partial}{\partial \kvec{i}} \hat{y}(K) =
X \wvec{i}
\]
So for $x\p \in S$ we have:
\[
x\p \cdot \frac{\partial}{\partial \kvec{i}} \hat{y} =
(x\p)^\top X \wvec{i}
=  \wvec{i}_0 \cdot f_{x\p} (X)
\]

Combining the above with the definition of the loss function,
$\ell_y(\hat{y}) = - y \hat{y}$,
and with \lemref{lem_vec} we get:
\begin{align*}
x\p \frac{\partial}{\partial \kvec{i}}
\mathcal{L}_{K_t,W_0} & = \mean{(X,y) \sim \mathcal{G}}
{x\p \frac{\partial}{\partial \kvec{i}}
\ell_y(\hat{y})} \\
& = \mean{(X,y) \sim \mathcal{G}}{-y\,x\p \frac{\partial}{\partial
  k_i} \hat{y}} \\
& = \mean{(X,y) \sim \mathcal{G}}{-y\, \wvec{i}_0 \cdot f_{x\p} (X)} \\
& = \wvec{i}_0 \cdot \mean{(X,y) \sim \mathcal{G}}{-y f_{x\p}(X)}  \\
& = \frac{1}{d} \wvec{i}_0 \cdot v_c
\end{align*}

\end{proof}

As an immediate corollary we obtain that a gradient step does not
change the projection of the kernel on two vectors that
correspond to the same class (both are in the same $S_c$). 
\begin{corollary}
\label{cor:symmetry}
For every $t\ge0$, $i \in [n]$,
for every semantic class
$c \in \mathcal{C}$ and for every
$x_1,x_2 \in S_c$ it holds that:
$|\kvec{i}_{t+1} \cdot x_1 - \kvec{i}_{t+1} \cdot x_2|
= |\kvec{i}_t \cdot x_1 - \kvec{i}_t \cdot x_2|$.
\end{corollary}
\begin{proof}
From \lemref{lem_grad} we can conclude that for a given
$c \in \mathcal{C}$, for every
$x_1,x_2 \in S_c$  we get:
\[
x_1 \frac{\partial}{\partial \kvec{i}}
\mathcal{L}_{K_t,W_0} = x_2 \frac{\partial}{\partial \kvec{i}}
\mathcal{L}_{K_t,W_0}
\]
From the gradient descent update rule:
\[
\kvec{i}_{t+1} = \kvec{i}_t - \eta \frac{\partial}{\partial \kvec{i}}
\mathcal{L}_{K_t,W_0}
\]
And therefore:
\begin{align*}
|\kvec{i}_{t+1} \cdot x_1 - \kvec{i}_{t+1} \cdot x_2|
&= |(\kvec{i}_t - \eta \frac{\partial}{\partial \kvec{i}}\mathcal{L}_{K_t,W_0})
\cdot x_1 -
(\kvec{i}_t - \eta \frac{\partial}{\partial \kvec{i}}\mathcal{L}_{K_t,W_0})
\cdot x_2| \\
&= |\kvec{i}_{t} \cdot x_1 - \kvec{i}_{t} \cdot x_2
- (\eta \frac{\partial}{\partial \kvec{i}}\mathcal{L}_{K_t,W_0} x_1 -
\eta \frac{\partial}{\partial \kvec{i}}\mathcal{L}_{K_t,W_0} x_2)| \\
&= |\kvec{i}_t \cdot x_1 - \kvec{i}_t \cdot x_2|
\end{align*}
\end{proof}

% separating different classes
Next we turn to show that a gradient step improves the separation of
vectors coming from different semantic classes. 
\begin{lemma}
\label{lem:diff_class}
Fix $c_1, c_2 \in \mathcal{C}$. Recall that we denote
$\angle(v_{c_1}, v_{c_2})$
to be the angle between the vectors $v_{c_1},v_{c_2}$. Then, with probability
$\angle(v_{c_1}, v_{c_2})/\pi$ on the initialization of $\wvec{i}_0$
we get:
\[
\sign(\wvec{i}_0 \cdot v_{c_1}) \ne \sign(\wvec{i}_0 \cdot v_{c_2})
\]
\end{lemma}
\begin{proof}
Observe the projection of $\wvec{i}_0$ on the plane spanned by
$v_{c_1}, v_{c_2}$. Then, the result is immediate from the symmetry
of the initialization of $\wvec{i}_0$.
\end{proof}

\begin{lemma}
\label{lem:diff_class_dist}
Fix $c_1 \neq c_2 \in \mathcal{C}$. Then, with probability of at least
$0.23\frac{\angle(v_{c_1}, v_{c_2})}{\pi}$
we get for every $x_1 \in S_{c_1}, x_2 \in S_{c_2}$:
\[
|\kvec{i}_T \cdot x_1 - \kvec{i}_T \cdot x_2| >
\frac{1}{d} \eta T \frac{\norm{v_{c_1}} + \norm{v_{c_2}}}{2} - 2 \sigma
\]
\end{lemma}

\begin{proof}
Notice that since $\wvec{i}_0 \sim \mathcal{N}(0,1)$,
we get that $\wvec{i}_0 \cdot v_{c_j} \sim \mathcal{N}(0,\norm{v_{c_j}}^2)$
for $j \in \{1,2\}$.
Therefore, the probability that $\wvec{i}_0 \cdot v_{c_j}$ deviates by at most
$\frac{1}{2}$-std from the mean is $\mathrm{erf}(\frac{1}{2\sqrt{2}})$.
Thus, we get that:
\[
P(|\wvec{i}_0 \cdot v_{c_j}| \le \norm{v_{c_j}}) = \mathrm{erf}(\frac{1}{2\sqrt{2}})
\]
And using the union bound:
\[
P(|\wvec{i}_0 \cdot v_{c_1}| \le \norm{v_{c_1}}
\vee 
|\wvec{i}_0 \cdot v_{c_2}| \le \norm{v_{c_2}}) \le 2\mathrm{erf}(\frac{1}{2\sqrt{2}}) < 0.77
\]
Thus, using \lemref{lem:diff_class}, we get that the following holds
with probability of at least
$0.23\frac{\angle(v_{c_1}, v_{c_2})}{\pi}$:
\begin{itemize}
\item $|\wvec{i}_0 \cdot v_{c_1}| > \norm{v_{c_1}}$
\item $|\wvec{i}_0 \cdot v_{c_2}| > \norm{v_{c_2}}$
\item $\sign(\wvec{i}_0 \cdot v_{c_1}) \ne \sign(\wvec{i}_0 \cdot v_{c_2})$
\end{itemize}
Assume w.l.o.g that
$\wvec{i}_0 \cdot v_{c_1} < 0 < \wvec{i}_0 \cdot v_{c_2}$,
then using \lemref{lem_grad} we get:
\begin{align*}
\kvec{i}_T x_1 &= \kvec{i}_0 x_1 -\eta \sum_{t=1}^T
\frac{\partial}{\partial \kvec{i}}x_1 \mathcal{L}_{K_t,W_0} \\
&= \kvec{i}_0 -\eta \sum_{t=1}^T
\frac{1}{d} \wvec{i}_0 \cdot v_{c_1}\\
&= \kvec{i}_0 -\frac{1}{d}\eta T \wvec{i}_0 \cdot v_{c_1}
> \frac{1}{d}\eta T \frac{\norm{v_{c_1}}}{2} - \sigma
\end{align*}
In a similar fashion we can get:
\[
\kvec{i}_T x_2 < -\frac{1}{d}\eta T \frac{\norm{v_{c_2}}}{2} + \sigma
\]
And thus the conclusion follows:
\[
\kvec{i}_T x_1 - \kvec{i}_T x_2 > \frac{1}{d} \eta T \frac{\norm{v_{c_1}} + \norm{v_{c_2}}}{2} - 2 \sigma
\]
\end{proof}

Finally, we are ready to prove the main theorem.

\begin{proof} of \thmref{thm:convergence}.\\
We show two things:
\begin{enumerate}
\item Fix $c \in \mathcal{C}$. By the initialization, we get that
for every $x_1, x_2 \in S_c$ and for every $i \in [n]$:
\[
|x_1 \cdot \kvec{i}_0 - x_2 \cdot \kvec{i}_0| < \frac{\sigma}{\sqrt{n}}
\]
Using \crlref{cor:symmetry}, we get that:
\[
|x_1 \cdot \kvec{i}_T - x_2 \cdot \kvec{i}_T| < \frac{\sigma}{\sqrt{n}}
\]
And thus:
\[
\norm{K_T x_1 - K_T x_2} < \sigma
\]
\item
Let $c_1 \ne c_2 \in \mathcal{C}$.
Assume $T > \frac{2 \sqrt{s} \gamma d}{\eta \lambda}$.
For $i \in [n]$,
from \lemref{lem:diff_class_dist} we get that with probability of at least
$0.23\frac{\angle(v_{c_1}, v_{c_2})}{\pi} > 
0.23\frac{\theta}{\pi}$
for every
$x_1 \in S_{c_1}, x_2 \in S_{c_2}$:
\[
|\kvec{i}_T x_1 - \kvec{i}_T x_2| > 
\frac{1}{d} \eta T \frac{\norm{v_{c_1}} + \norm{v_{c_2}}}{2} - 2 \sigma
> \frac{1}{d} \eta T \lambda - 2 \sigma
> 2 \sqrt{s} \sigma
\]

For a given $c_1 \ne c_2 \in \mathcal{C}$, denote the event: 
\[
A_{c_1,c_2} = \{\forall i \in [n]: ~|\kvec{i}_T \cdot x_1 - \kvec{i}_T \cdot x_2| < 2 \sqrt{s} \sigma, ~ x_1 \in S_{c_1}, x_2 \in S_{c_2}\}
\]
Then, from what we have showed, it holds that:
\[
P(A_{c_1,c_2}) <
(1-0.23\frac{\theta}{\pi})^n \le \exp (-0.23n\frac{\theta}{\pi})
\]
Using the union bound, we get that:
\[
P(\exists c_1 \ne c_2 \in \mathcal{C}~ s.t~ A_{c_1,c_2})<
\exp (-0.23n\frac{\theta}{\pi}) |\mathcal{C}|^2
\]
Choosing
$n > \frac{2\pi}{0.23\theta} \log (\frac{|\mathcal{C}|}{\delta})$
we get $P(\exists c_1 \ne c_2 \in \mathcal{C}~ s.t~ A_{c_1,c_2}) < \delta$.
Now, if for every $c_1 \ne c_2 \in \mathcal{C}$ the event $A_{c_1,c_2}$
doesn't hold, then clearly for every $x_1 \in S_{c_1},x_2 \in S_{c_2}$
we would get $\norm{K_T \cdot x_1 - K_T \cdot x_2} > 2 \sqrt{s} \sigma$,
and this is what we wanted to show.
\end{enumerate}
\end{proof}

\section{Proofs of \lemref{lem:orth_mapping} and of \thmref{thm:full_alg}} \label{app:proof_of_orth_mapping}

\begin{proof}[of \lemref{lem:orth_mapping}]
by induction:
\begin{itemize}
\item the case $\kappa = k$ follow immediately, since by the choice of
  $\gamma$, all the observed patches are with distance of at least
  $2\gamma$, and by definition of the clustering algorithm, we get
  that $\phi_k$ is an orthogonal patch mapping, and since $h_k = id$
  we get the required.
\item assume the claim holds for $\kappa+1$ and we prove that it also
  holds for $\kappa$. Let $\varphi_{\kappa+1}$ be the mapping satisfying
the condition of the claim for $\kappa+1$.
Notice that the data that is fed to the two-layer training step
comes from the distribution $\varphi_{\kappa+1} \ast \mathcal{G}_{\kappa+1}$, and satisfies
the conditions for the analysis in the previous section.
Now, define the map
$\varphi_\kappa: \mathcal{C}_{\kappa}^s \rightarrow \reals^\ell$ in
the following way:
first, for every patch $p \in \mathcal{C}_{\kappa}^s$
we take an arbitrary manifestation of the patch $p$
in the next level, denoted $z \in \mathcal{C}_{\kappa+1}^{s^2}$.
In other words, $z$ could be any $s^2$ sub-image in the next level
that could be generated from the patch $p$. Now,
take $\varphi_\kappa(p) := \phi_\kappa (K_\kappa \cdot (\varphi_{\kappa+1} \ast z))$.
Notice the following is true with probability at least $1-\delta\p$, 
using \thmref{thm:convergence}:
\begin{enumerate}
\item $\varphi_\kappa(p)$ does not depend on the choice of $z$:
if $z,z\p$ are two different manifestations of $p$,
then, from the definition of the generative model,
for every $i \in [s]$ it holds that $z_i, z\p_i \in S_{p_i}$.
Thus from what we have shown:
$\norm{K_{\kappa} \cdot \varphi_{\kappa+1}(z_i)-K_{\kappa} \cdot \varphi_{\kappa+1}(z\p_i)} < \sigma
= \frac{\gamma}{\sqrt{s}}$.
Therefore:
\begin{align*}
\norm{K_\kappa \cdot \varphi_{\kappa+1} \ast z -K_\kappa \cdot \varphi_{\kappa+1} \ast z\p}^2 &=
\sum_{i=1}^s \norm{K_\kappa \cdot \varphi_{\kappa+1}(z_i) -K_\kappa \cdot \varphi_{\kappa+1}(z\p_i)}^2  \\
&<\gamma^2 
\end{align*}
and thus by the properties of the clustering algorithm:
\[
\phi_\kappa (K_\kappa \cdot (\varphi_{\kappa+1} \ast z)) =
\phi_\kappa (K_\kappa \cdot (\varphi_{\kappa+1} \ast z\p))
\]
\item for every two patches $p \ne p\p$ we get $\varphi_\kappa(p) \perp \varphi_\kappa(p\p)$:
let $z,z\p$ be the manifestations of $p, p\p$ respectively.
Since $p \ne p\p$ there exists $i \in [s]$ such that $p_i \ne p\p_i$.
From the generative model it follows that $z_i \in S_{p_i}$ and
$z\p_i \in S_{p\p_i}$, and therefore from the behavior of the algorithm:
$\norm{K_{\kappa}^\top \cdot \varphi_{\kappa+1}(z_i)-K_{\kappa}^\top \cdot \varphi_{\kappa+1}(z\p_i)} > 2\sqrt{s}\sigma = 2\gamma$.
Therefore, we get that $\norm{K_\kappa \cdot \varphi_{\kappa+1} \ast z -K_\kappa \cdot \varphi_{\kappa+1} \ast z\p} > 2\gamma$, and thus
from the clustering algorithm we get $\varphi_\kappa(p) \perp \varphi_\kappa(p\p)$.
\end{enumerate}
Now, recall that in the algorithm definition:
$h_\kappa = K_\kappa^\top \cdot (\phi_{\kappa+1} \ast h_{\kappa+1})$.
Using the assumption for $\kappa + 1$, we get:
\[
h_\kappa \circ \mathcal{G} =
K_\kappa (\varphi_{\kappa+1} \ast \mathcal{G}_{\kappa+1})
\]
and from the definition of $\varphi_\kappa$ and what we have shown we get:
\[
\phi_\kappa \ast (h_{\kappa} \circ \mathcal{G})
= \varphi_\kappa \ast \mathcal{G}_\kappa
\]
Therefore, the required holds for $\kappa$.
\end{itemize}
\end{proof}

\begin{proof}[ of \thmref{thm:full_alg}]
From \lemref{lem:orth_mapping}, after performing $k$ iterations
of the algorithm, we observe the distribution $\varphi_1 \ast \mathcal{G}_1$,
where $\varphi_1$ is an orthonormal patch mapping.
Again, from \thmref{thm:convergence} we get that the two-layer algorithm
learns $K_0$ such that for every two patches $p,p\p \in \mathcal{C}_1^s$:
if $p,p\p$ are from the same class
we get that $\norm{K_0 \phi (p) - K_0 \phi(p\p)} < \sigma < \gamma$,
and if $p,p\p$ are from different classes
we get that $\norm{K_0 \phi (p) - K_0 \phi(p\p)} > 2\sqrt{s}\sigma = 2 \gamma$.
By the clustering algorithm, we will get that $\phi_0$
maps these patches to orthonormal vectors. Thus, we feed $\mathrm{CLS}$
with data from the distribution $\mathcal{G}_0$, where the classes of 
$\mathcal{C}_0$ are mapped to orthonormal vectors. The proof follows by
the assumptions about the $\textrm{CLS}$
algorithm. 
\end{proof}

\newpage

\section{Parameters for Experiments}
\label{app:params}
\figref{fig:params} below lists the parameters that were learned in the random
parameter search for the different configurations of the algorithm,
as described in \ref{fig:results}.
The table
lists the parameters used in each layer: $\ell_1,\ell_2$ are the
number of clusters for the first and second layer, and
$\ell\p_1,\ell\p_2$ are the output channels of the Conv1x1 operation
for each layer.
These parameters could be used to reproduce the results of our experiments.
\begin{figure}[H]
\footnotesize
\center
\begin{tabular}{|l|l|l|l|l|l|l|l|}
  \hline
   Classifier & $N$ & $k_1$ & $k\p_1$ & $k_2$ & $k\p_2$ & $b$ & Accuracy \\
   Ours+FC & 47509 & 1377 & 155 & 3534 & 216  & -0.14 & \textbf{0.734} \\
   Ours+Linear & 32124 & 1384 & 97 & 2576 & 211 & -0.63 & 0.689 \\
   \hline
   Clustering+JL+FC & 12369 & 39 & 39 & 184 & 184 &  0.84 & 0.586 \\
   Clustering+JL+Linear & 57893 & 5004 & 345 & 6813 & 407 & -0.01 & 0.588 \\   
  \hline
\end{tabular}
\caption{Parameters used in our experiment}\label{fig:params}.
\end{figure}

\end{document}